\documentclass{article}

\usepackage{microtype}
\usepackage{graphicx}
\usepackage{caption}
\usepackage{subfigure}
\usepackage{svg} 
\usepackage{amssymb}
\usepackage{comment}
\usepackage{enumitem}
\usepackage{pifont}
\usepackage{booktabs} 

\usepackage{hyperref}

\usepackage[accepted]{icml2026}

\usepackage{amsmath}
\usepackage{amssymb}
\usepackage{mathtools}
\usepackage{amsthm}

\usepackage[capitalize,noabbrev]{cleveref}

\theoremstyle{plain}
\newtheorem{theorem}{Theorem}[section]

\theoremstyle{definition}

\newtheorem{assumption}[theorem]{Assumption}
\theoremstyle{remark}

\usepackage[textsize=tiny]{todonotes}

\icmltitlerunning{MARVL: Multi-Stage Guidance for Robotic Manipulation via Vision-Language Models}

\begin{document}

\twocolumn[
\icmltitle{MARVL: Multi-Stage Guidance for Robotic Manipulation via Vision-Language Models}



\icmlsetsymbol{equal}{*}

\begin{icmlauthorlist}
\icmlauthor{Xunlan Zhou}{equal,njuis,lamda}
\icmlauthor{Xuanlin Chen}{equal,njuis,lamda}
\icmlauthor{Shaowei Zhang}{lamda,njuai}
\icmlauthor{ShengHua Wan}{lamda,njuai}
\icmlauthor{Xiaohai Hu}{uw}
\icmlauthor{Lei Yuan}{lamda,njuai}
\icmlauthor{De-chuan Zhan}{lamda,njuai}
\end{icmlauthorlist}

\icmlaffiliation{njuis}{School of Intelligent Science and Technology, Nanjing University, China}
\icmlaffiliation{lamda}{National Key Laboratory for Novel Software Technology, School of Artificial Intelligence, Nanjing University, China}
\icmlaffiliation{njuai}{School of Artificial Intelligence, Nanjing University, China}
\icmlaffiliation{uw}{MACS Lab, University of Washington}

\icmlcorrespondingauthor{De-chuan Zhan}{zhandc@lamda.nju.edu.cn}

\icmlkeywords{Machine Learning, ICML}

\vskip 0.3in
]



\printAffiliationsAndNotice{\icmlEqualContribution} 


\begin{abstract} Designing dense reward functions is pivotal for efficient robotic Reinforcement Learning (RL). However, most dense rewards rely on manual engineering, which fundamentally limits the scalability and automation of reinforcement learning. While Vision-Language Models (VLMs) offer a promising path to reward design, naïve VLM rewards often misalign with task progress, struggle with spatial grounding, and show limited understanding of task semantics. To address these issues, we propose \textbf{MARVL}—\textbf{M}ulti-st\textbf{A}ge guidance for \textbf{R}obotic manipulation via \textbf{V}ision-\textbf{L}anguage models. MARVL fine-tunes a VLM for spatial and semantic consistency and decomposes tasks into multi-stage subtasks with task direction projection for trajectory sensitivity. Empirically, MARVL significantly outperforms existing VLM-reward methods on the Meta-World benchmark, demonstrating superior sample efficiency and robustness on sparse-reward manipulation tasks. \end{abstract}

\section{Introduction}

Reinforcement learning (RL) has achieved notable progress in robotic manipulation \citep{levine2016end,kalashnikov2018scalable,gu2017deep,han2023survey,kroemer2021review,tang2025deep,singh2022reinforcement}, yet reward specification remains a bottleneck \citep{yu2025reward,ibrahim2024comprehensive, fu2024robot, kaelbling1996reinforcement, li2017deep, eschmann2021reward, lee2026roboreward}. Vision-Language Models (VLMs) offer a promising path toward reward design by grounding rewards in natural language \citep{ma2024vision, rocamonde2023vision,venuto2024code,lee2026roboreward,sontakke2023roboclip,wang2024rl, venkataraman2024real}. A widely adopted approach computes the cosine similarity between visual embeddings of environment observations and text embeddings of task instructions, using this score as a dense reward signal \citep{rocamonde2023vision,adeniji2023language, baumli2023vision, ma2023liv, furl}. This strategy leverages the rich knowledge of pre-trained VLMs, but its effectiveness in guiding RL remains limited.


\begin{figure}[t!]
    \centering
    \includegraphics[width=\linewidth]{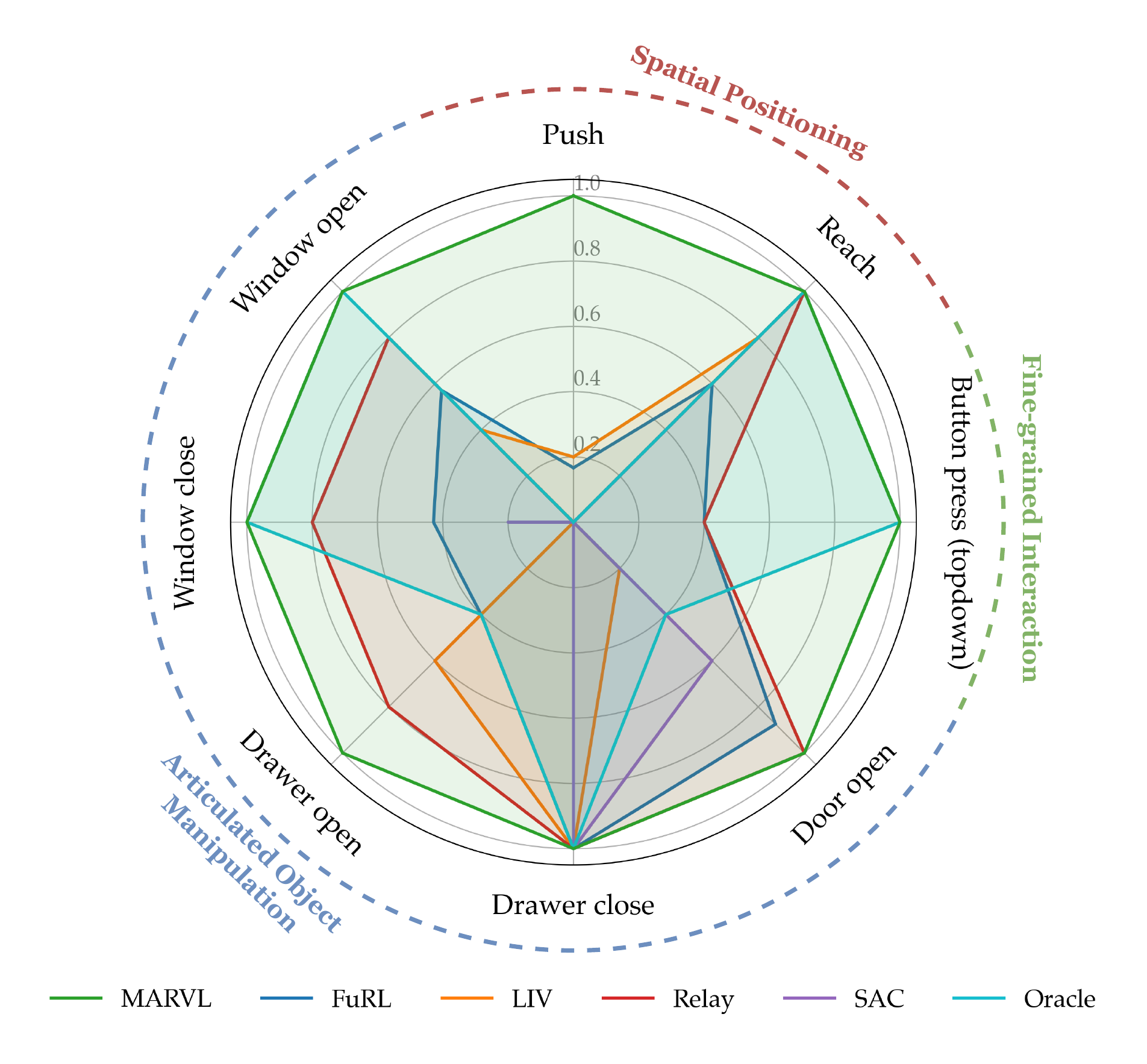}
    \caption{Radar plot of performance across eight Meta-World manipulation tasks. MARVL achieves consistently strong and balanced performance across all skill categories, surpassing the Oracle reward on several tasks and outperforming prior VLM-based reward methods.}
    \label{fig:fig1}
    \vspace{-3mm}
\end{figure}

Despite the advancements introduced in these varied studies \citep{rocamonde2023vision,adeniji2023language, baumli2023vision, ma2023liv, furl, huang2024dark}, we find that VLM-based rewards continue to face three critical challenges: \textbf{(i) Weak Spatial Grounding}: VLMs struggle with cross-view consistency and are highly sensitive to camera perspectives, leading to misaligned rewards. \textbf{(ii) Lack of Progress Awareness}: embeddings are easily distracted by irrelevant visual factors, failing to reflect incremental progress and exposing agents to reward hacking. \textbf{(iii) Semantic Misalignment}: VLMs frequently misinterpret task instructions, leading to unreliable reward signals. These limitations suggest that current progress in VLM-guided RL may be less robust than initially believed.


To address these challenges, we propose \textbf{MARVL}—\textbf{M}ulti-st\textbf{A}ge guidance for \textbf{R}obotic manipulation via \textbf{V}ision-\textbf{L}anguage models, a plug-and-play framework for RL algorithms. MARVL incorporates three key design components:
(1) \textbf{Scene-View Decomposition}: Disentangles task semantics from viewpoint noise via targeted fine-tuning to restore \textit{spatial grounding};
(2) \textbf{Task Direction Projection}: Projects embeddings onto instruction-defined vectors, transforming volatile similarities into monotonic signals for \textit{progress awareness};
(3) \textbf{Confidence-Thresholded Shaping}: Gates rewards based on semantic confidence, filtering spurious matches to mitigate \textit{semantic misalignment}.

Extensive experiments on the Meta-World benchmark demonstrate that MARVL substantially improves success rate and robustness over all VLM-reward baselines as shown in Figure \ref{fig:fig1}. The primary contributions of this work are as follows:



\begin{itemize}[leftmargin=*, itemsep=0pt]
    \item We analyze existing VLM reward paradigms and identify three key limitations: weak spatial grounding, lack of task progress sensitivity, and semantic misalignment.
    \item We propose \textbf{MARVL}, a plug-and-play framework that improves the quality and reliability of VLM rewards through targeted refinements.
    \item We demonstrate that MARVL consistently outperforms prior methods across diverse environments.
\end{itemize}

\section{Related Works}

\subsection{Markov Decision Process and RL}
Reinforcement learning (RL) provides a principled framework for sequential decision making under uncertainty, and is typically formalized as a Markov Decision Process (MDP) \citep{puterman2014markov,sutton2018reinforcement, li2017deep}. An MDP is defined by the tuple $(\mathcal{S}, \mathcal{A}, P, r, \gamma)$, where $\mathcal{S}$ denotes the state space, $\mathcal{A}$ the action space, $P(s'|s,a)$ the transition dynamics, $r(s,a)$ the reward function, and $\gamma \in [0,1)$ a discount factor. At each timestep $t$, the agent observes a state $s_t \in \mathcal{S}$, selects an action $a_t \in \mathcal{A}$, and transitions to the next state $s_{t+1}$ according to $P$, while receiving a scalar reward $r(s_t,a_t)$. The objective of RL is to learn a policy $\pi(a|s)$ that maximizes the expected discounted return $\mathbb{E}[\sum_{t=0}^{\infty} \gamma^t r(s_t,a_t)]$. In practice, this policy is typically modelled using a neural network $\pi_{\theta}(a|s)$ with learnable parameters $\theta$ \cite{mnih2013playing}. This formulation has proven effective in domains ranging from robotics to game playing, yet the design of the reward function $r$ remains a central challenge, particularly in tasks where dense supervision is difficult to obtain \cite{yu2025reward, eschmann2021reward, dewey2014reinforcement}.

\subsection{Vision-Language Models (VLMs) and RL}

Recent progress in vision-language models (VLMs) has demonstrated their ability to bridge vision and natural language through joint embedding spaces \citep{radford2021clip,alayrac2022flamingo, jia2021scaling,li2022blip}. A landmark example is CLIP \citep{radford2021clip}, which learns aligned representations of texts and images, enabling strong zero-shot transfer across diverse vision tasks. While such models were originally developed with general vision applications in mind, growing attention has been directed toward their integration with reinforcement learning (RL). VLMs provide rich semantic grounding that can benefit RL in multiple ways, including reward specification and shaping \citep{mahmoudieh22a,rocamonde2023vision,adeniji2023language,baumli2023vision,lubana2023fomo,dang2023clip,huang2024dark,furl}, task description and success verification \citep{du2023vision}, and state representation learning \citep{chen2024vision}. These directions highlight the potential of VLMs as powerful priors that enhance vision-based RL, offering richer representations and more flexible reward signals for complex tasks.

\begin{figure*}[t!]  
  \centering
  \includegraphics[width=\textwidth]{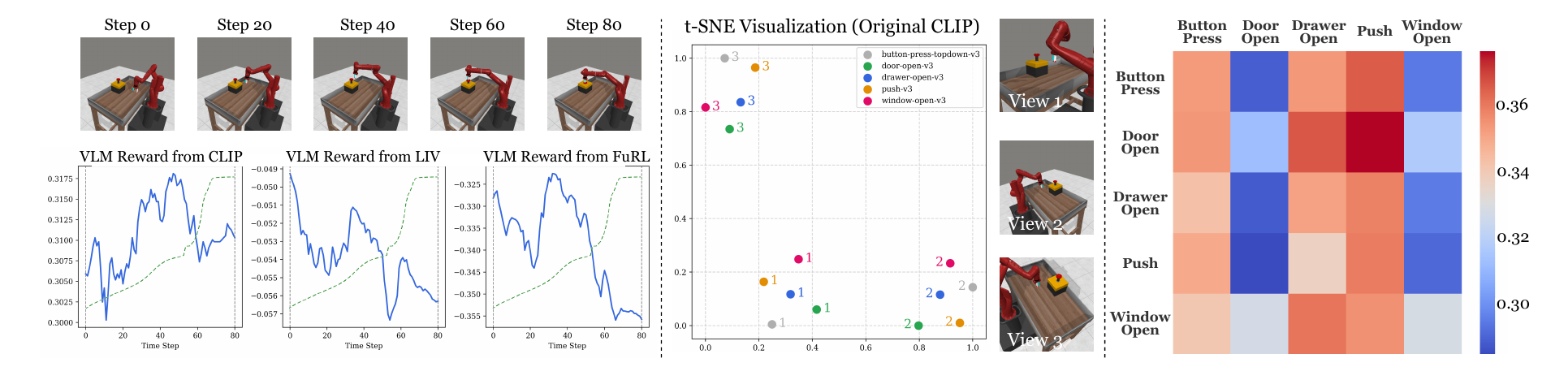} 
  \caption{\textbf{Reward Misalignment in VLM-Based Methods.} \textbf{Left:} VLM reward signals along an oracle \texttt{button-press-topdown} trajectory. The green dashed curve denotes the environment-provided dense reward in Meta-World, whose scale differs from VLM rewards and is shown only to indicate the overall trend of task progress. \textbf{Middle}: t-SNE projection of image embeddings from five Meta-World tasks under three viewpoints, showing that embeddings cluster by viewpoint rather than by task identity.        
\textbf{Right}: Text–image similarity matrix for the instruction set 
(\texttt{button-press-topdown}, \texttt{door-open}, \texttt{drawer-open}, \texttt{push}, and \texttt{window-open}) in original CLIP, where rows correspond to image embeddings and columns correspond to text embeddings. 
Ideally, the diagonal should dominate, but weak diagonal structure and noisy off-diagonal activations indicate poor alignment between text and image embeddings.}
  \label{fig:problem}
  \vspace{-2mm}
\end{figure*}

\subsection{VLMs as Rewards}


Leveraging Vision-Language Models (VLMs) as a source of reward has emerged as a rapidly growing direction in reinforcement learning \citep{mahmoudieh22a, rocamonde2023vision, adeniji2023language, baumli2023vision, lubana2023fomo, dang2023clip, nam2023lift,furl, venuto2024code,lee2026roboreward,sontakke2023roboclip,wang2024rl}. Some approaches directly treat VLM-based similarity scores between observations and task instructions as the primary reward signal \citep{mahmoudieh22a, rocamonde2023vision, adeniji2023language}, while others use VLM-derived rewards as an additional source of supervision to complement the environment’s original reward \citep{rocamonde2023vision, baumli2023vision, lubana2023fomo, dang2023clip,furl}. Together, these studies highlight the growing interest in integrating the rich knowledge encoded in VLMs into reinforcement learning reward design.

A key motivation behind this line of research is that many RL tasks naturally come with sparse rewards, making it difficult for agents to explore effectively. Given an observation $s_t$ at timestep $t$, the agent samples an action $a_t \sim \pi_\theta(a_t|s_t)$ and receives a task reward $r^{\text{task}}_t$ upon execution. In common sparse settings, this reward is defined as $r^{\text{task}}_t = \delta_{\text{success}}$, i.e., $1$ only upon task success and $0$ otherwise. Such sparse feedback is ubiquitous in practice but significantly hinders efficient training.


To mitigate this challenge, recent works propose augmenting the sparse task reward with an auxiliary signal derived from VLMs \citep{chen2024vision,rocamonde2023vision,furl}. Concretely, the reward is reshaped as:

\begin{equation}
r_t = r^{\text{task}}_t + \rho \cdot r_t^\text{vlm}, \label{eq:main}
\end{equation}

where $\rho$ is a scalar balancing the task reward and the VLM-based reward $r_t^\text{vlm}$. A common choice for $r_t^\text{vlm}$ is the CLIP score, defined as the cosine similarity between the text embedding of the language goal and the image embedding of the current observation:

\begin{equation}
r_t^\text{vlm} \triangleq r_t^\text{CLIP} = \frac{\langle e_{l}, e_{o_t} \rangle}{\lVert e_l \rVert \cdot \lVert e_{o_t} \rVert},
\end{equation}

where $o_t$ denotes the observation at step $t$, $l$ denotes the language instruction, and $e_{l}$, $e_{o_t}$ denote the embedding of the instruction and the observation \citep{radford2021clip}.

\subsection{Challenges in VLM Rewards}


Despite its simplicity, the CLIP-based reward signal is often criticized for its fuzziness, as existing literature \citep{furl,lubana2023fomo,huang2024dark,mahmoudieh22a,rocamonde2023vision} demonstrates that such signals are frequently compromised by semantic misalignment and false positives, failing to provide the robust feedback required for complex tasks. For instance, high similarity scores can be triggered by spurious visual–text matches that are unrelated to task progress, leading the agent toward ineffective exploration. A concrete example of such noisy and uninformative reward dynamics is shown in Figure~\ref{fig:problem} left.

To address these reliability issues, recent work has explored differing strategies. One direction focuses on suppressing irrelevant noise. BIMI \cite{huang2024dark} was introduced to mitigate false positives, and VLM-RMs-GB \cite{rocamonde2023vision} proposed Goal-Baseline Regularization to project out extraneous information. Despite their contributions, BIMI has not effectively solved the false positive problem, and VLM-RMs-GB in general lacks the adaptability required for complex tasks. Alternatively, approaches like LIV \cite{ma2023liv} and FuRL \cite{furl} attempt to fine-tune the reward representations directly through value-oriented objectives or alignment. LIV trains embeddings with a value-oriented objective, and can optionally be fine-tuned on tailored datasets for domain adaptation, while FuRL applies lightweight alignment to the reward space. However, as also illustrated in Figure~\ref{fig:problem}, these approaches remain imperfect.

While these studies provide valuable insights, their analyses remain fragmented: some highlight noise and instability \citep{huang2024dark, lubana2023fomo, rocamonde2023vision}, others note misalignment \cite{furl, ma2023liv}, but none offer a comprehensive perspective. In this work, we identify three fundamental challenges underlying VLM rewards—\textbf{weak spatial grounding}, \textbf{lack of task progress sensitivity}, and \textbf{semantic misalignment}—which we discuss systematically in this
section.

\begin{figure*}[t!]  
  \centering
  \includegraphics[width=\textwidth]{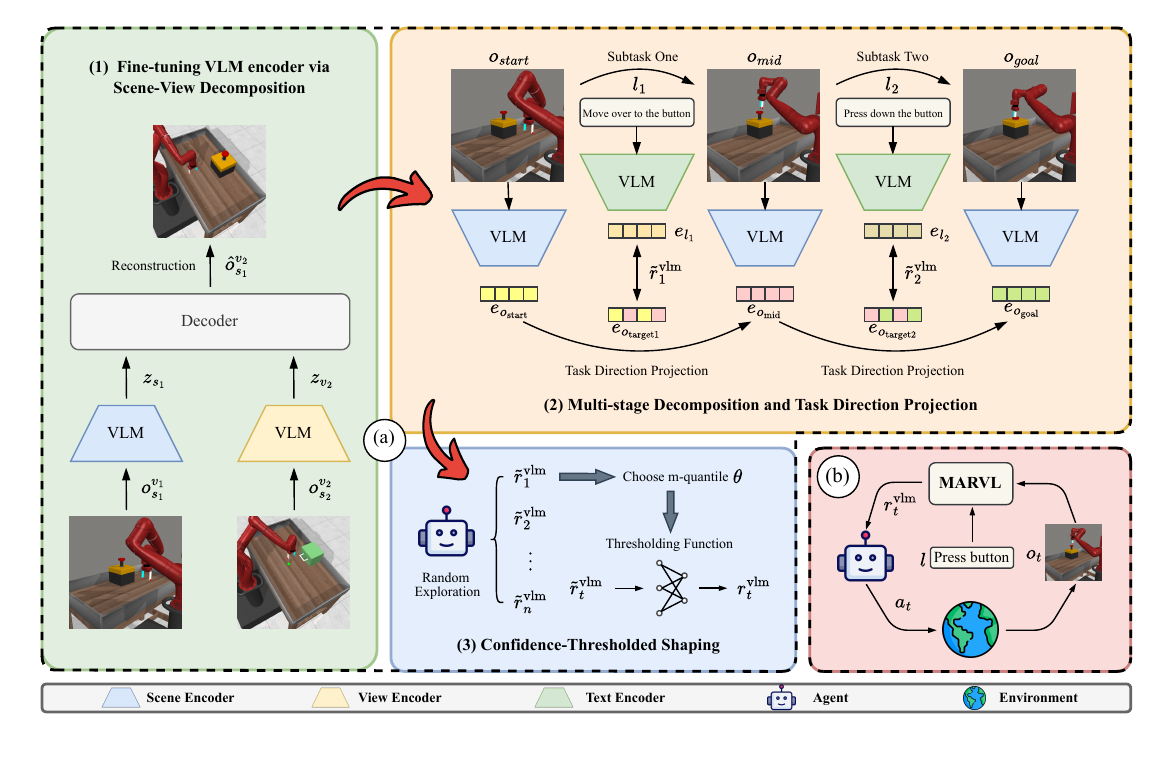} 
  \caption{\textbf{Overview of MARVL.} 
(a) We first fine-tune the VLM encoder via Scene-View Decomposition, followed by Multi-stage Decomposition with Task Direction Projection and Confidence-Thresholded Shaping to derive calibrated VLM rewards. 
(b) MARVL could be easily integrated into the RL loop, where the agent interacts with the environment under dense VLM reward guidance.}
  \label{fig:overview}
\end{figure*}

\textbf{Weak Spatial Grounding.} Ideally, VLM-based rewards should capture task semantics invariant to camera perspectives \citep{ma2023liv, pang2025reviwo, liu2024robouniview, zhou2023semantically}. However, our analysis reveals that VLM representations are dominated by superficial visual cues: embeddings from diverse Meta-World tasks cluster by \textit{camera viewpoint} rather than \textit{task identity} (see Figure~\ref{fig:problem}, middle). This lack of viewpoint invariance indicates that the visual space fails to ground instructions in the physical scene. Consequently, the reward signal tracks incidental observational shifts instead of actual agent progress, fundamentally destabilizing text--image alignment.

\textbf{Lack of Progress Awareness.} While minor fluctuations are expected, a dense reward signal should exhibit a strong positive correlation with task progress \citep{furl, ma2023liv, huang2024dark}. However, as shown in Figure~\ref{fig:problem} (left), VLM rewards along an oracle trajectory display excessive volatility, failing to capture the global trend of task completion. The signal frequently oscillates or drops even as the agent approaches the goal. This poor signal-to-noise ratio indicates that the VLM is insensitive to fine-grained physical states, rendering the reward unreliable for guiding policy learning.


\textbf{Semantic Misalignment.} Leveraging VLMs for reward design assumes a well-aligned latent space where instructions naturally map to their corresponding observations \citep{furl,huang2024dark,lubana2023fomo,sontakke2023roboclip,ma2023liv}. 
However, Figure~\ref{fig:problem} (right) contradicts this assumption: the heatmap exhibits diffuse activations with only a weak diagonal, indicating poor task separability in the embedding space. Consistently, Figure~\ref{fig:problem} (left) shows that even after adaptations such as LIV and FuRL, the reward remains strictly negative, suggesting a persistent representation misalignment. 
In this setting, the reward provides only a weak ordinal signal---ranking relative similarities among poorly aligned instances---rather than a calibrated measure of semantic correctness.

\section{Methods}

\subsection{Fine-tuning VLM via Scene-View Decomposition}

To mitigate the view-sensitivity of off-the-shelf VLMs, we employ a \textbf{Scene-View Decomposition} framework inspired by \cite{pang2025reviwo}. We leverage an automatically collected dataset of \textbf{500} multi-view Meta-World images  to adapt the VLM's visual encoder toward robust spatial grounding.

\textbf{Disentangled Architecture.} As illustrated in Figure~\ref{fig:overview}(a), we adopt a dual-encoder--decoder structure initialized from CLIP. Given an image $x$, the model produces two disentangled codes:
\begin{itemize}[leftmargin=*, itemsep=0pt]
    \item \textbf{Scene representation} $z_s = E_s(x)$: captures view-invariant, task-critical semantics (e.g., robot and object states). $E_s$ is retained as the backbone for our downstream reward model.
    \item \textbf{View representation} $z_v = E_v(x)$: captures nuisance variations like camera pose. This auxiliary branch is discarded after training.
\end{itemize}

\begin{figure}
    \centering
    \includegraphics[width=\linewidth]{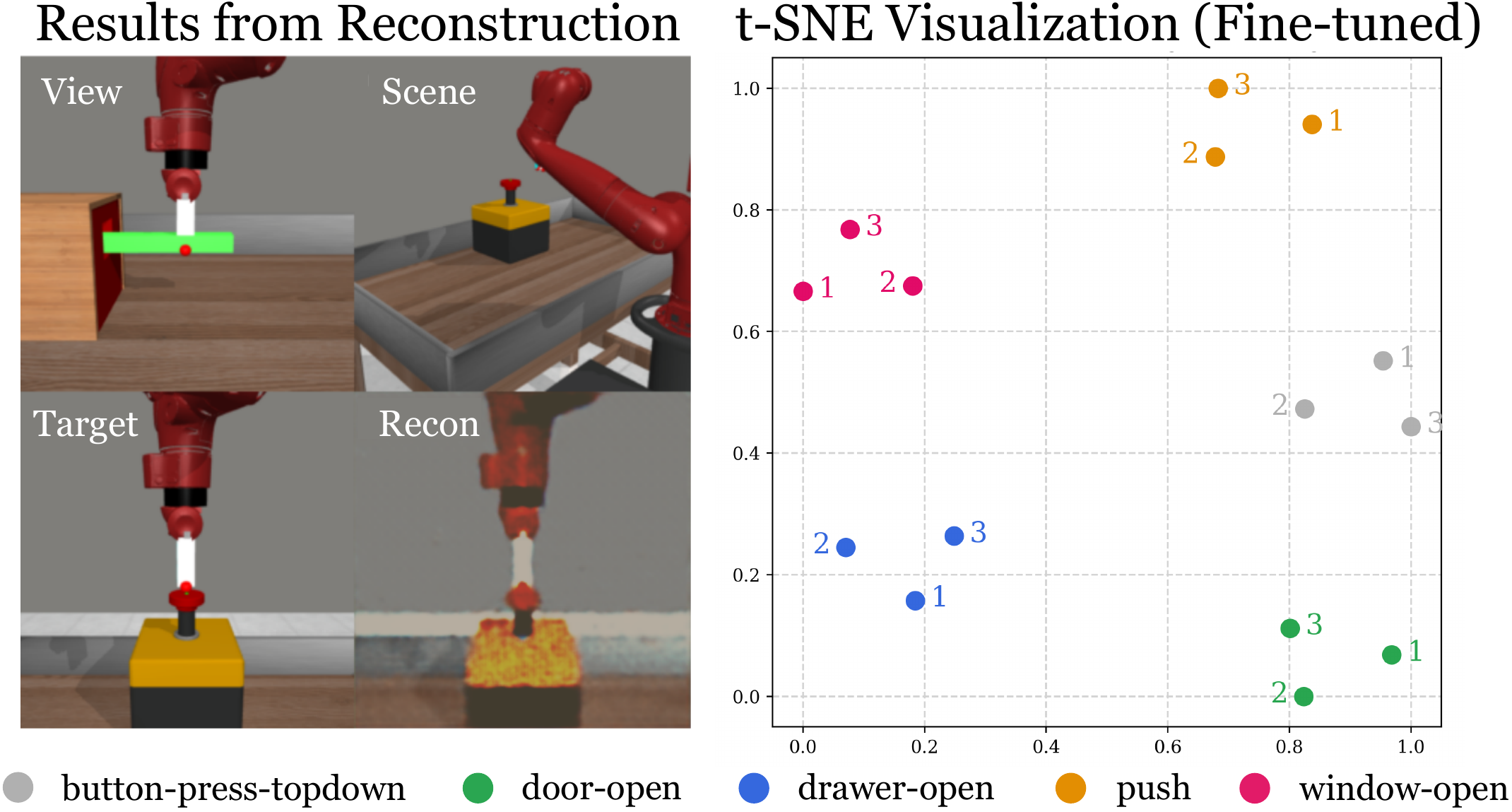}
    \caption{\textbf{Effectiveness of Scene-View Decomposition.} 
\textbf{Left:} Reconstructions retain critical spatial structure (e.g., arm pose) despite high-frequency pixel abstraction, verifying the encoder's geometric grounding. 
\textbf{Right:} t-SNE visualization confirms that fine-tuned embeddings cluster by \textit{task identity} rather than camera viewpoint, demonstrating effective disentanglement of semantics from nuisance factors.}
    \label{fig:ft}
    \vspace{-2mm}
\end{figure}

To enforce disentanglement, we employ a reconstruction objective based on latent swapping. Given two observations $o_{s_1}^{v_1}$ and $o_{s_2}^{v_2}$ that differ in both scene and viewpoint, we combine the scene code of the first ($z_{s_1}$) with the view code of the second ($z_{v_2}$) to synthesize a cross-compositional reconstruction $\hat{o}_{s_1}^{v_2} = D(z_{s_1}, z_{v_2})$. This forces $z_s$ to isolate spatial semantics from perspective shifts. The total loss is defined as:
\begin{equation}
    \mathcal{L} = \mathcal{L}_{\text{recon}} + \lambda_1 \mathcal{L}_{\text{shuffle}} + \lambda_2 \mathcal{L}_{\text{consistency}} + \lambda_3 \mathcal{L}_{\text{clip}}.
\end{equation}
We provide full definitions, mathematical formulations, and fine-tuning hyperparameters for each loss term in Appendix~C.

As shown in Figure~\ref{fig:ft}, this process successfully realigns the embedding space: scene embeddings cluster by task identity rather than viewpoint, providing a stable, spatially grounded foundation for MARVL.

\subsection{Task Direction Projection via Multi-Stage Decomposition}

To bridge the gap between high-level instructions and low-level control, we adopt a Multi-stage Decomposition strategy (Figure~\ref{fig:overview}(a)), which breaks a long-horizon task into a sequence of subtasks $\{l_k\}$ with corresponding intermediate goal states. This decomposition localizes progress evaluation to short horizons, where dense reward signals are most critical for effective policy learning.

However, off-the-shelf VLM embeddings are high-dimensional and dominated by task-irrelevant variations, causing raw similarity-based rewards to fluctuate and violate monotonicity even within a single stage. To address this issue, we introduce \textbf{Task Direction Projection (TDP)}, a geometric mechanism that suppresses orthogonal noise while preserving task-aligned progress signals.

\textbf{Task Direction Principle.}
Rather than assuming linear evolution of physical trajectories or embeddings, we rely on a weaker condition: for each subtask, there exists a task-aligned direction in the joint visual–semantic embedding space such that progress along this direction provides a monotonic proxy for task completion after noise attenuation. This direction is defined by the displacement between the start and goal embeddings and is used as a reference axis for filtering, not as a literal model of task execution.

Formally, given a task direction vector $\mathbf{d}$, we define a projection operator
\begin{equation}
    \mathcal{P}_{\mathbf{d}}(\mathbf{x}) 
    =
    \left(
    \alpha \frac{\mathbf{d}\mathbf{d}^\top}{\lVert \mathbf{d} \rVert^2}
    +
    (1-\alpha)\mathbf{I}
    \right)\mathbf{x},
    \quad \alpha \in [0,1].
\end{equation}
This operator preserves the component of $\mathbf{x}$ aligned with $\mathbf{d}$ while attenuating orthogonal components by a factor $(1-\alpha)$.
As shown in Appendix~D, this anisotropic filtering substantially improves the signal-to-noise ratio of the embedding trajectory and restores monotonicity of the induced reward with respect to a projected progress coordinate, without assuming linearity in the underlying physical dynamics. We apply TDP to both modalities as follows.

\begin{figure}[t!]
  \centering
  \includegraphics[width=\linewidth]{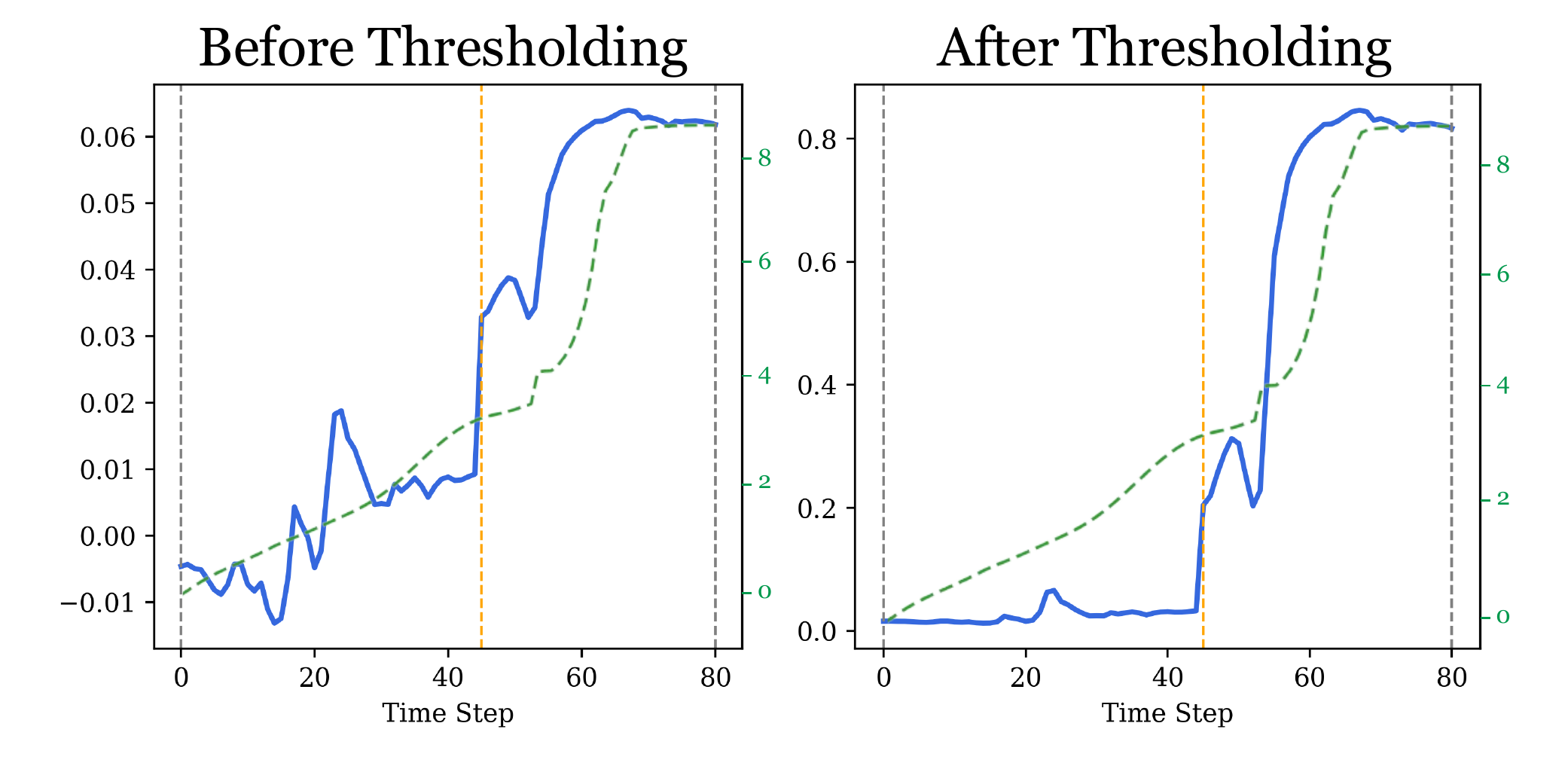}
 \caption{Raw vs. thresholded MARVL rewards on the \texttt{button}-\texttt{press}-\texttt{topdown} task using oracle trajectories. 
The raw MARVL reward (left) exhibits noise and spurious positive activations at early steps, while the thresholded reward (right) suppresses these false positives and amplifies high-confidence signals, yielding a cleaner and more reliable reward signal under MARVL. The orange dashed line indicates the stage transition point. The green dashed curve shows the environment-provided dense reward.}
  \label{fig:thresholding}
  \vspace{-2mm}
\end{figure}

\textbf{Text Projection.} We isolate the dynamic ``action" semantics from static context. Given a subtask $l_k$ (e.g., ``robot press button") and a baseline $l_b$ (e.g., ``robot arm"), we define the text direction $\mathbf{d}_{\text{text}} = e_{l_k}- e_{l_b}$. The refined target $e'_{l_k} = \mathcal{P}_{\mathbf{d}_{\text{text}}}(e_{l_k})$ thus focuses strictly on the manipulation intent.

\begin{figure*}[t!]
  \centering
  \includegraphics[width=\textwidth]{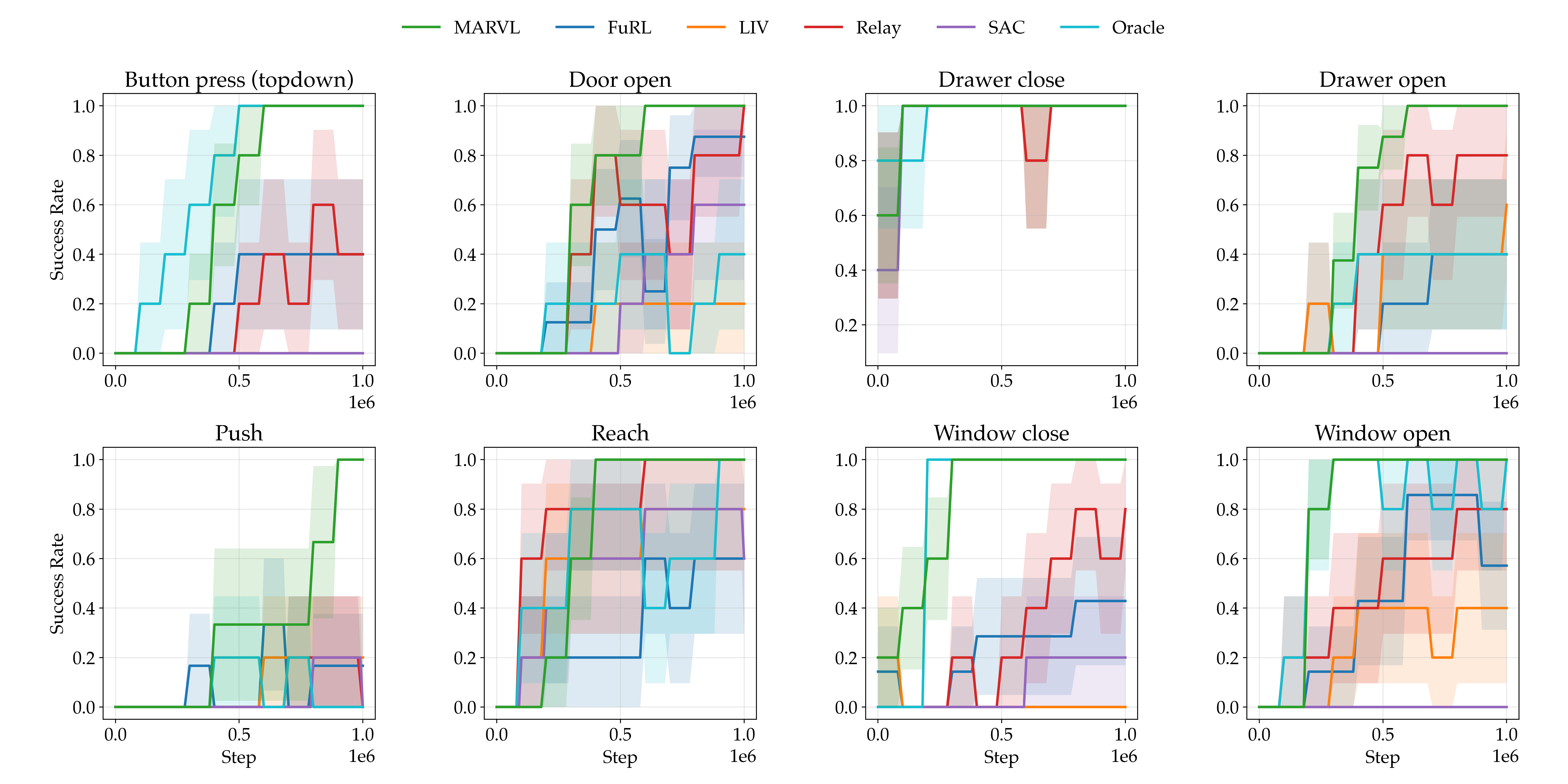}
  \caption{Learning curves on eight Meta-World manipulation tasks. We report the mean success rate across 5 random seeds. MARVL (ours) is compared against state-of-the-art VLM-reward baselines (FuRL, LIV, Relay) and an Oracle agent trained with the environment’s standard hand-crafted dense reward using internal states. MARVL consistently outperforms prior VLM methods and achieves performance parity with the privileged Oracle, demonstrating that learned visual rewards can effectively match the quality of engineered dense rewards.}
  \label{fig:mainexp}
  \vspace{-2mm}
\end{figure*}

\textbf{Image Projection.} 
Given the start observation $o_{\text{start}}$ and the subtask goal observation $o_{\text{goal}}$, we define the visual task direction 
$\mathbf{d}_{\text{img}} = e_{o_{\text{goal}}} - e_{o_{\text{start}}}$. 
The projected embedding of the current observation $o_t$ is defined as:
\begin{equation}
    e'_{o_t} = \mathcal{P}_{\mathbf{d}_{\text{img}}}(e_{o_t}-e_{o_{\text{start}}}) + e_{o_{\text{start}}}.
\end{equation}
This formulation grounds the reward in a projected progress coordinate defined by the start and goal embeddings, 
suppressing background variations while preserving task-aligned semantic change.

\textbf{Reward and Transition.} The dense reward is computed as the cosine similarity between the projected embeddings: $\tilde{r}_t = \cos(e'_{o_t}, e'_{l_k})$. Stage transitions are automatically managed by monitoring the similarity between the current observation and the intermediate target embedding. Once this similarity surpasses a predefined threshold, the system recognizes the subtask as complete and advances the agent to the next stage. In our experiments, intermediate target images are manually collected; however, Appendix~B.3 demonstrates that they can be generated fully automatically using LLMs and image generation models.

\subsection{Confidence-Thresholded Shaping}

In VLM-based rewards, irrelevant actions often produce small but consistent positive scores due to generic semantic similarity. As a result, agents can accumulate these trivial signals instead of making genuine task progress.

To rectify this, we introduce \textbf{Confidence-Thresholded Shaping (CTS)}. We first quantify the environment's background noise by collecting scores from the first 10k steps, setting the threshold $\theta$ as the $m$-th quantile. We then apply a soft gating function:
\begin{equation}
    r^{\text{vlm}}_t = \sigma\big(\kappa (\tilde{r}^{\text{vlm}}_t - \theta)\big).
\end{equation}
This mechanism serves a \textbf{dual purpose}: it effectively \textbf{zeros out low-magnitude noise} to prevent spurious accumulation, while simultaneously \textbf{amplifying high-confidence matches} via the steepness parameter $\kappa$. As illustrated in Figure~\ref{fig:thresholding}, this sharpens the reward gradient for semantically relevant states, blocking trivial solutions while providing a decisive, high-contrast signal once the agent breaks through the noise floor. We further explore more advanced filtering schemes---such as adaptive EMA with hysteresis and Kalman filtering---which are discussed in Appendix~E.3.

\section{Experiments}  

In this section, we investigate the following questions: (1) How does MARVL perform compared to existing VLM-reward baselines? (2) Are the main components of MARVL effective individually? (3) Can MARVL generalize across different RL backbones? (4) Does MARVL transfer well to new environments and camera settings?


\subsection{Experiment Setup}

We conduct our main evaluation on the \textbf{Meta-World} benchmark \cite{yu2020meta}, a standard suite of vision-based robotic manipulation tasks. We select eight environments spanning diverse behaviors, including pushing, button pressing, door opening, and drawer manipulation. Meta-World is widely used in prior work on VLM-based rewards and visual manipulation, enabling fair comparison with existing methods, and its tasks are semantically well-defined and exhibit clear notions of incremental progress, making them particularly suitable for evaluating language-guided reward design. To assess cross-domain generalization, we further evaluate MARVL on the \textbf{Panda-Gym} benchmark \cite{gallouedec2021panda}, which differs from Meta-World in both robot embodiment and visual appearance.

\subsection{Baselines}

To validate the effectiveness of MARVL, we compare it against the following baselines. All methods use the same state-based SAC backbone \cite{haarnoja2018soft} and differ only in how rewards are computed. The hyperparameters for the SAC are attached in Appendix B.


\textbf{1. SAC} \cite{haarnoja2018soft}. A standard SAC agent using only the sparse binary task reward $r^{\text{task}}_t$. This represents the lower bound of performance given no additional shaping signals.

\textbf{2. LIV} \citep{ma2023liv}. A state-based SAC agent trained
with task reward and dense LIV reward as in Eqn. \ref{eq:main}.

\textbf{3. FuRL} \citep{furl}. A recent VLM-based reward method that leverages two key mechanisms: lightweight reward alignment and relay RL, where a VLM-guided agent and an auxiliary SAC agent take turns acting in the environment and share a replay buffer so that the SAC agent can help the VLM agent escape local minima induced by fuzzy VLM rewards.

\textbf{4. Relay} \citep{lan2023can}. A simplified variant of FuRL that keeps only the relay RL mechanism used  in FuRL.

\textbf{5. MARVL}. Our full method including Scene-View Decomposition,  Task Direction Projection, and confidence-thresholded reward shaping.

\begin{figure}[t]
    \centering
    \includegraphics[width=\linewidth]{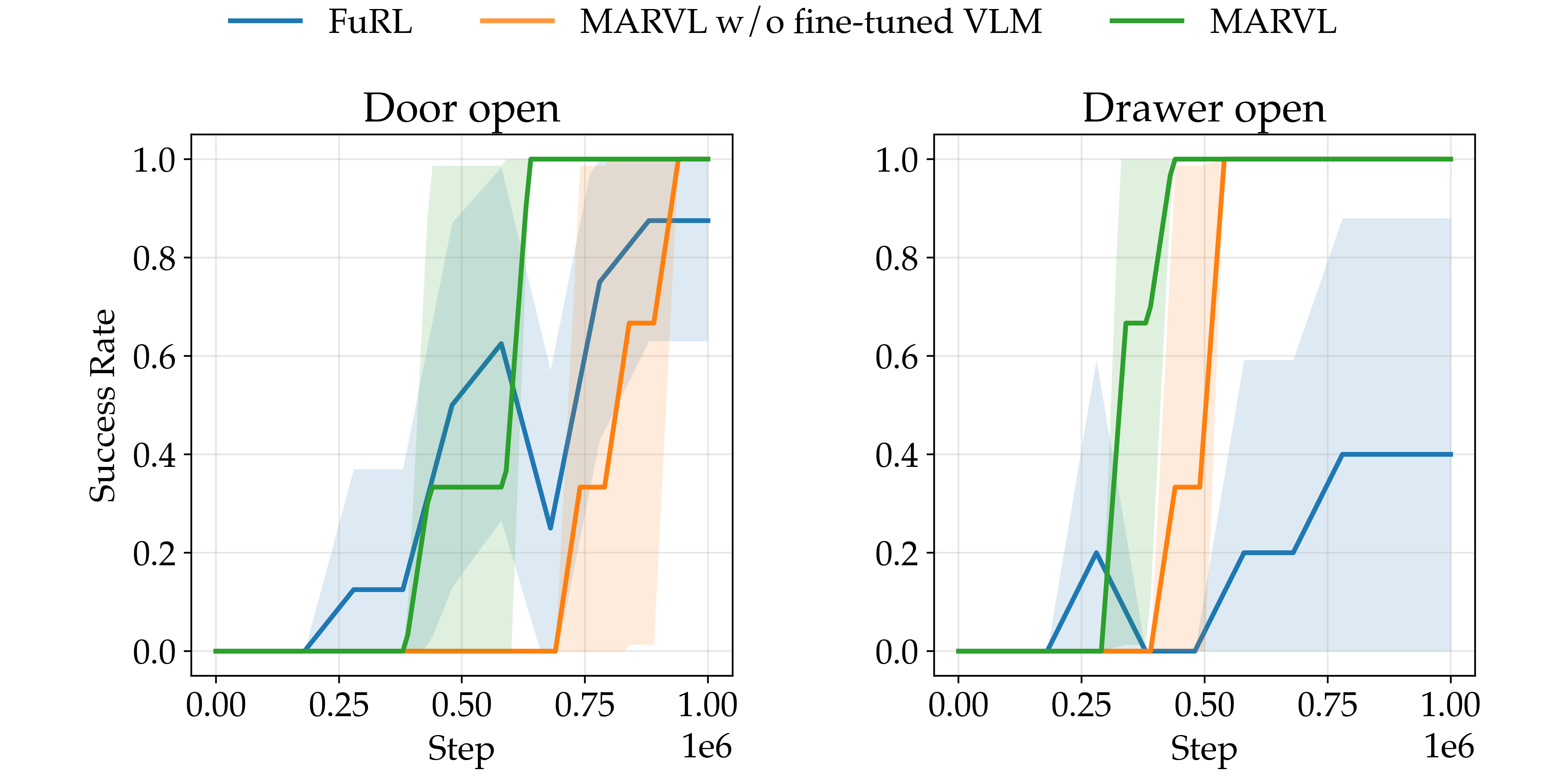}
    \caption{\textbf{Impact of Scene-View Decomposition}}
    \label{fig:w/ovlm}
\vspace{-2mm}
\end{figure}

\begin{figure}[t]
\vspace{-1mm}
    \centering
    \includegraphics[width=\linewidth]{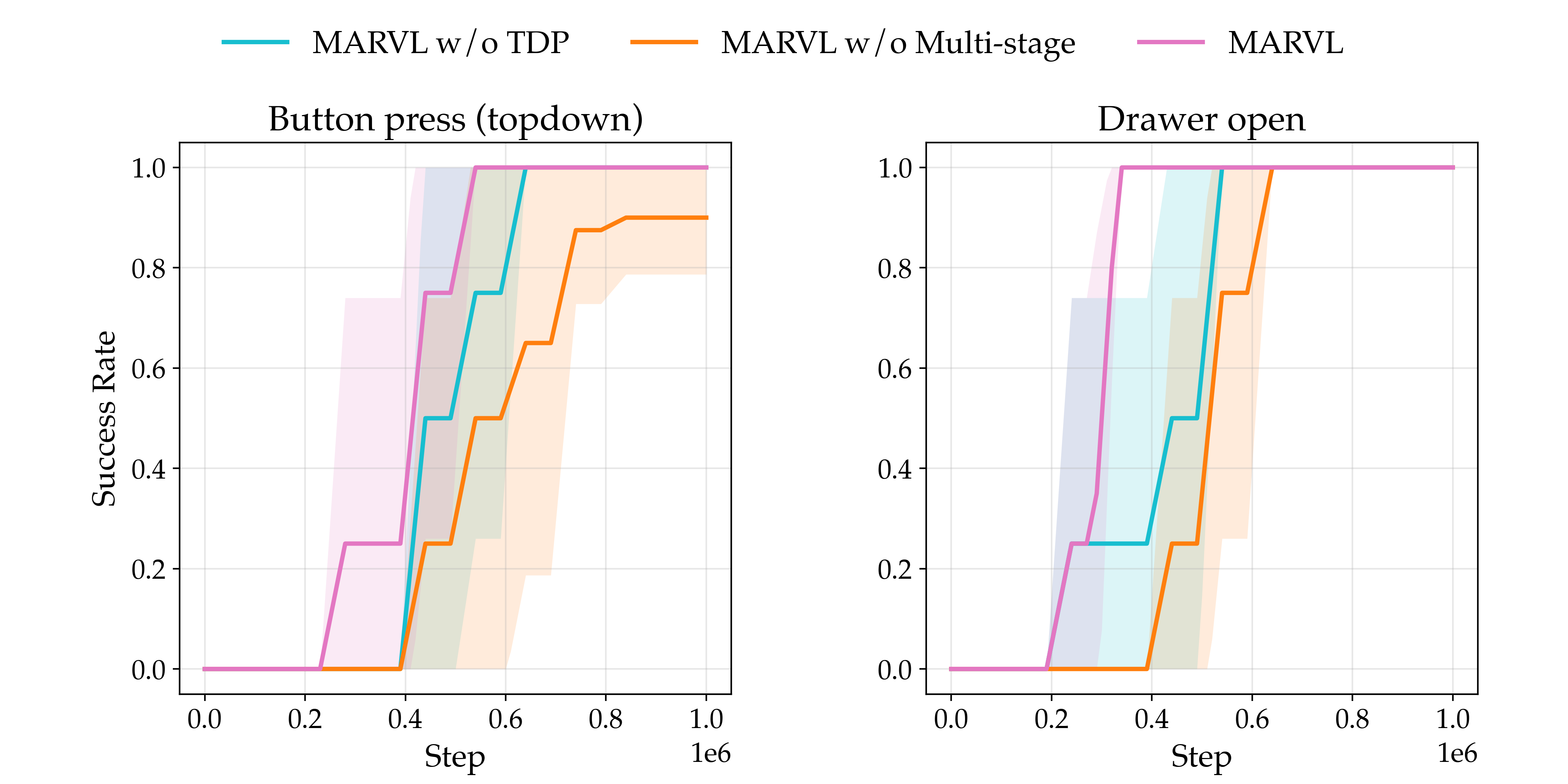}
    \caption{\textbf{Impact of TDP and Multi-Stage Decomposition}}
    \label{fig:TDP}
      \vspace{-2mm}
\end{figure}

\subsection{Main Results on Meta-World Tasks}



Figure~\ref{fig:mainexp} summarizes the performance of MARVL across eight sparse-reward tasks, compared against prior VLM-based methods and an Oracle baseline. Here, the Oracle corresponds to an agent using the dense rewards provided by Meta-World, which are manually designed from environment states. Across all tasks, MARVL consistently achieves higher success rates and faster convergence than existing VLM-based rewards. While methods such as FuRL and Relay improve exploration over standard RL, they remain sensitive to noisy and unstructured visual signals, often converging to suboptimal plateaus. 

We further compare MARVL against the Meta-World dense Oracle rewards. Despite relying solely on learned visual–semantic signals, MARVL matches the Oracle’s performance on tasks such as Button Press and Window Close, and exceeds it in exploration-heavy environments, including Push and Door Open. This indicates that semantically grounded rewards can provide even more effective guidance than handcrafted dense rewards in complex, sparse-reward settings.

\subsection{Ablation Studies}
In this section, we conduct a comprehensive analysis to investigate the internal mechanisms and external robustness of MARVL.

\textbf{Impact of Scene-View Decomposition.}
Removing Scene-View Decomposition and relying on raw, pre-trained VLM encoders consistently degrades performance (Figure~\ref{fig:w/ovlm}), indicating that separating view-dependent variation from scene-level semantics improves reward stability. Further analysis is provided in Appendix~E.2.

\textbf{Impact of Task Direction Projection and Multi-Stage Decomposition.}
Ablating either Task Direction Projection (TDP) or Multi-Stage Decomposition preserves final success but significantly slows convergence (Figure~\ref{fig:TDP}), highlighting MARVL’s improved sample efficiency.

\textbf{Impact of Confidence-Thresholded Shaping.}
Removing Confidence-Thresholded Shaping and directly using raw VLM scores degrades both sample efficiency and training stability (Figure~\ref{fig:w/othr}). Thresholding accelerates convergence in Door Open and prevents performance collapse in Button Press, underscoring its role in suppressing spurious VLM signals.

\begin{figure}[t]
    \centering
    \includegraphics[width=\linewidth]{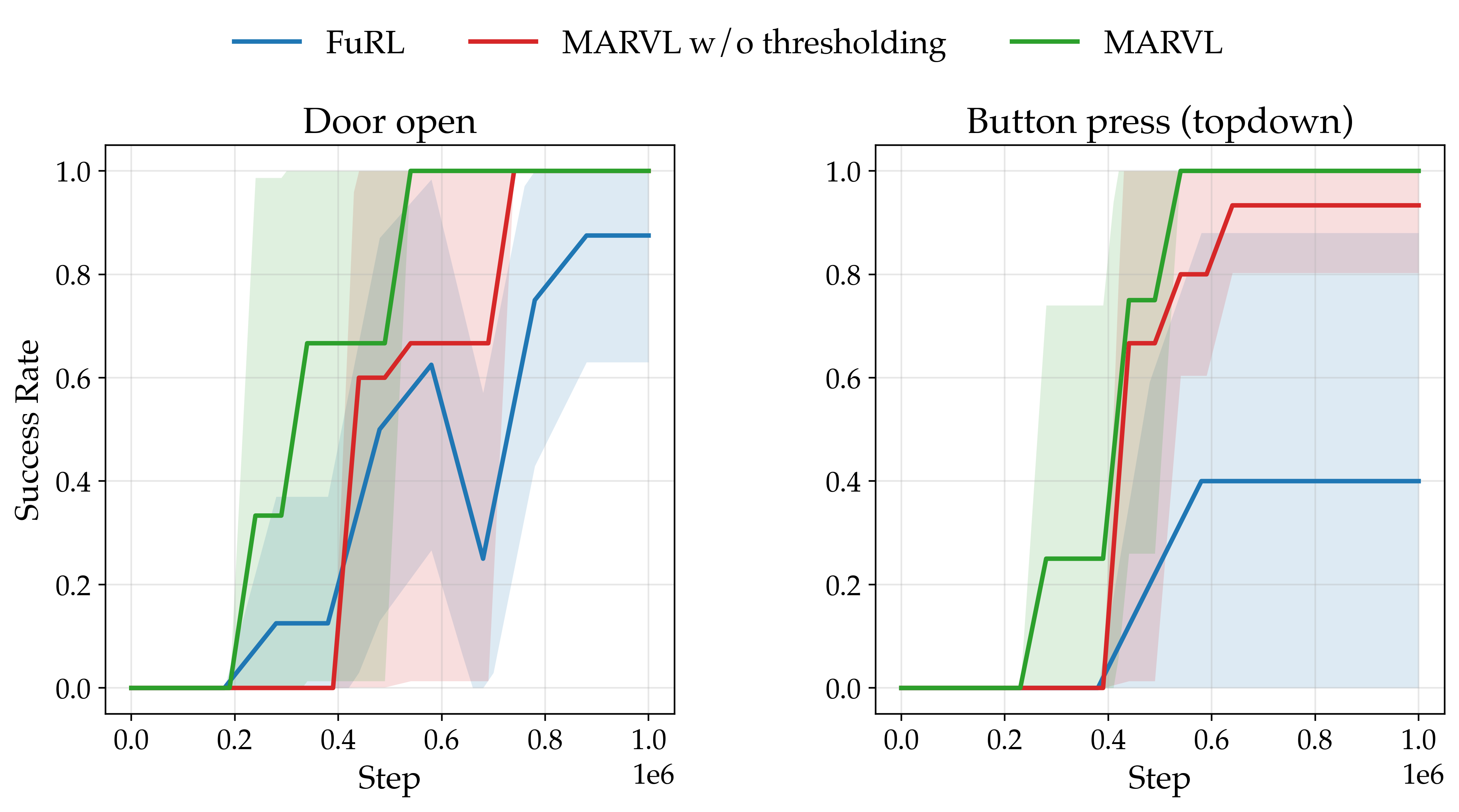}
    \caption{\textbf{Impact of Confidence-Thresholded Shaping}}
    \label{fig:w/othr}
        \centering
    \includegraphics[width=\linewidth]{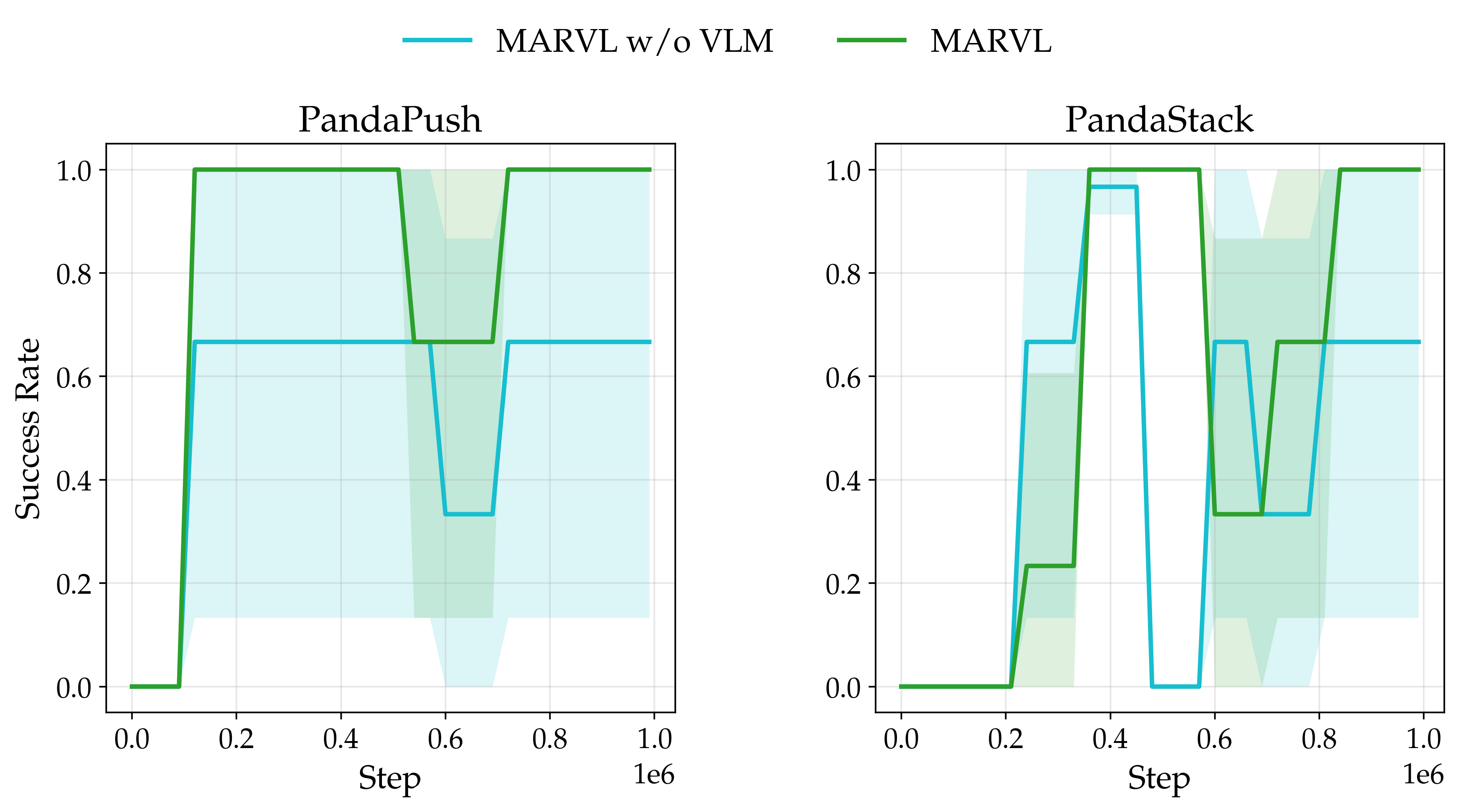}
    \caption{\textbf{Generalization to Novel Environments}}
    \label{fig:pandagym}
\end{figure}

\textbf{Generalization to Novel Environments.}
We evaluate cross-domain generalization to examine whether the proposed VLM fine-tuning improves transferable spatial grounding, rather than overfitting to Meta-World-specific visual statistics.
To this end, we deploy MARVL on the Panda-Gym benchmark, which differs substantially from Meta-World in both robot embodiment and visual appearance.
As shown in Figure~\ref{fig:pandagym}, MARVL maintains high success rates without any target-domain adaptation, while the variant using a non-fine-tuned VLM exhibits a severe performance degradation.
This result suggests that the fine-tuning procedure does not merely adapt the model to the training domain, but instead enhances its ability to capture transferable spatial semantics that generalize across environments and embodiments.

\textbf{Robustness to Diverse Camera Configurations.}
We assess robustness to viewpoint changes by evaluating MARVL under three distinct camera configurations. As shown in Figure~\ref{fig:camera}, performance remains consistent across views for both Button Press and Drawer Open tasks, with all settings converging to high success rates despite minor variance during early training. This suggests that MARVL is invariant to viewpoint-specific visual cues.

\textbf{Generalization across RL Backbones.}
To test algorithmic robustness, we replace the SAC backbone with TD3~\cite{fujimoto2018td3} while keeping the reward shaping unchanged. As shown in Figure~\ref{fig:td3}, MARVL maintains strong sample efficiency and substantially outperforms sparse-reward baselines, indicating that MARVL functions as an algorithm-agnostic reward interface across off-policy RL methods.

\begin{figure}[t]
    \centering
    \includegraphics[width=\linewidth]{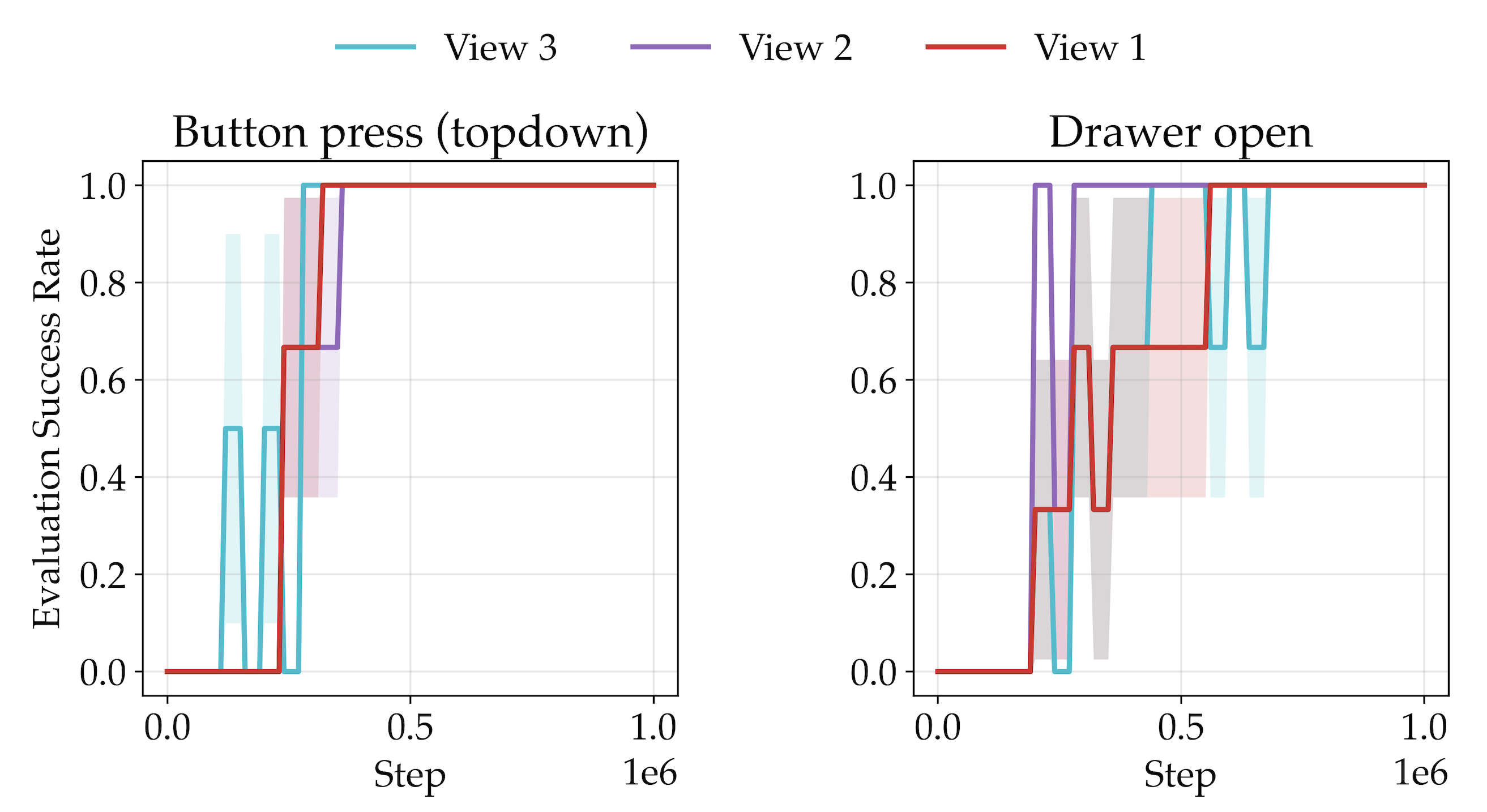}
    \caption{\textbf{Generalization to Camera Settings}}
    \label{fig:camera}
    \vspace{1.5mm}
        \centering
    \includegraphics[width=\linewidth]{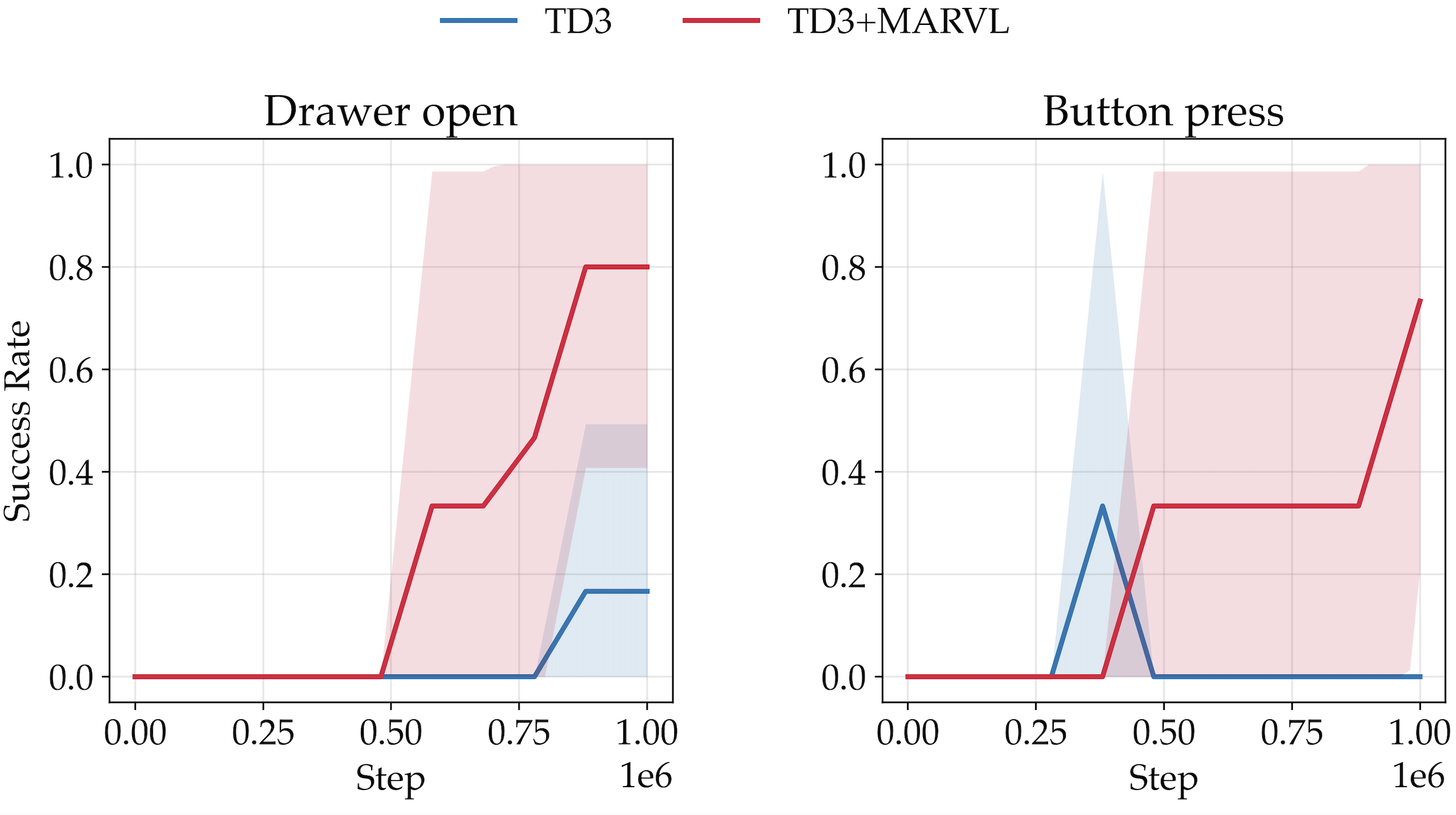}
    \caption{\textbf{Generalization across RL Backbones}}
    \label{fig:td3}
\end{figure}
\section{Conclusions}
This work shows that VLM-based rewards can become substantially more effective when their representations are explicitly grounded and structured. By integrating Scene-View Decomposition, Task Direction Projection, and Confidence-Thresholded Shaping, MARVL transforms raw VLM similarities into stable, progress-sensitive reward signals, leading to improved sample efficiency and robustness across diverse robotic manipulation tasks. These results indicate that the limitations of prior VLM-based rewards stem less from insufficient semantic knowledge and more from a lack of structural alignment.

Despite its strong performance, MARVL has several limitations that suggest directions for future work. The current framework relies on lightweight VLM fine-tuning and is validated in simulation, and extending it to strictly zero-shot, real-world deployment remains an important objective. In addition, scaling MARVL to longer-horizon and multi-object manipulation tasks is a promising direction. Overall, MARVL demonstrates that modest adaptation and principled structuring can substantially improve the practicality of VLM-based reinforcement learning.

\bibliography{example_paper}
\bibliographystyle{icml2026}

\newpage
\appendix
\onecolumn
\section{Broader Impact and Ethical Considerations}
\textbf{Societal Impact and Democratization.} MARVL significantly lowers the technical barrier for robotic programming by replacing complex reward engineering with natural language instructions. Such democratization facilitates the deployment of intelligent assistants in unstructured environments (e.g., household service, elderly care), potentially alleviating labor shortages in critical sectors. However, as with broadly applicable advancements in automation, widespread adoption necessitates ongoing discourse regarding workforce adaptation to mitigate potential displacement in manual labor markets.

\textbf{Safety and Ethical Considerations.} The safety of our system is intrinsically bound to the reliability of the underlying Vision-Language Models (VLMs). Since MARVL leverages pre-trained encoders, it may inherit perceptual hallucinations where the model is affected by yield false-positive success signals, posing risks in safety-critical physical environments. Consequently, real-world deployment must be accompanied by robust safety constraints and human-in-the-loop verification. Furthermore, foundation models often contain inherent biases regarding cultural contexts or object typologies, these must be rigorously audited to prevent performance disparities across diverse visual settings.
\section{Experimental Setup}
\subsection{Meta-World and Panda-Gym Benchmark}

We employ two established robotic manipulation and reinforcement learning benchmarks: \textbf{Meta-World} \cite{yu2020meta} and \textbf{Panda-Gym} \cite{gallouedec2021panda} to evaluate the generalizability of MARVL across diverse task types and embodiments. \textbf{Meta-World}, built on MuJoCo \cite{todorov2012mujoco}, features a 7-DoF Sawyer arm with a 4-dimensional continuous action space. We select 8 representative tasks (hidden-goal) to investigate complex manipulation capabilities, specifically focusing on spatial positioning, fine-grained interaction and articulated object manipulation. Furthermore, \textbf{Panda-Gym}, based on PyBullet, employs the Franka Emika Panda robot (7-DoF), which presents significant differences in kinematics, geometrical appearance, and rendering style compared to the Sawyer arm used in Meta-World. We use it to assess robustness against visual domain shifts and novel embodiments. 

\subsection{Hyperparameters}
Table ~\ref{tab:hyperparameters} outlines the key hyperparameters for the backbone algorithm, MARVL, the baseline methods followed from prior work, Table ~\ref{tab:finetune} details the specific settings for VLM fine-tuning. To ensure a fair comparison, We adopt the FuRL codebase \url{https://github.com/fuyw/FuRL} and adhere to the original baseline implementations and hyperparameters. 

\subsection{Multi-Stage Decomposition and Automation}



Figure~\ref{fig:multistage} illustrates the multi-stage text descriptions (with ``robot arm'' as the baseline text) and the corresponding target images used in our experiments. 
While these intermediate goals can be specified manually for controlled comparison, the proposed multi-stage decomposition does not rely on human supervision.
As shown in Figure~\ref{fig:ai}, the entire pipeline can be fully automated using an AI-based generation process, without requiring expert demonstrations or manually predefined task stages.
Specifically, we employ Qwen-3-Flash to rewrite a task instruction into sequential sub-goals, and Qwen-ImageEdit-Plus to synthesize the corresponding visual targets.

Beyond the Button Press task shown in Figure~\ref{fig:ai}, we further evaluate the fully automated pipeline across a broader set of Meta-World environments.
The success rate curves are summarized in Figure~\ref{fig:ai_compa}, where MARVL with AI-labeled intermediate goals consistently achieves competitive performance across diverse manipulation tasks.
These results indicate that the proposed automation strategy generalizes beyond a single task and can serve as a viable substitute for manually specified sub-goals in practice.

\textbf{Importantly, this decomposition introduces structural guidance rather than additional semantic supervision}: consistent with the basic VLM reward approach, both the human-labeled and AI-labeled variants are derived from the same original task instruction, differing only in how the instruction is restructured into intermediate goals.
Counter-intuitively, our experimental results show that the AI-labeled pipeline achieves even faster convergence and superior asymptotic performance compared to the human-labeled baseline constructed from oracle simulation frames.
We hypothesize that this advantage arises because AI-generated goal images are synthesized directly from the textual sub-goals, and are therefore naturally aligned with the language instructions by construction, avoiding the potential semantic ambiguity or incidental visual details present in human-selected frames.

Notably, we observe that the AI-labeled variant exhibits higher variance in training dynamics compared to its human-labeled counterpart.
We attribute this increased fluctuation primarily to the quality and consistency of the generated goal images: unlike oracle frames sampled from the simulator, AI-synthesized images may occasionally contain visual artifacts or subtle geometric inconsistencies, which introduce noise into the reward signal.
Despite this variability, the overall performance remains robust, suggesting that MARVL is resilient to moderate imperfections in intermediate goal specification.

Despite these promising results, we emphasize that the effectiveness of the automated pipeline is ultimately bounded by the quality of the generative models.
Severe visual hallucinations or physically implausible generations could mislead the reward signal and degrade policy learning.
Ensuring the physical plausibility and controllability of AI-generated intermediate goals thus remains an important direction for future work.

\subsection{Computational Resources}
All RL experiments (including baselines and MARVL) were conducted on a workstation equipped with NVIDIA GeForce RTX 3090 GPU and Dual Intel Xeon Silver 4210R CPUs (totaling 20 physical cores and 40 threads) , while the VLM fine-tuning was performed on NVIDIA A100 GPU with 28 CPU cores and 256GB RAM. Regarding runtime, the Scene-View Decomposition fine-tuning requires approximately 2 hours as a one-time cost, MARVL averages 1 hour per task (1M environment steps).

\begin{figure}[ht!]
  \centering
  \includegraphics[width=\linewidth]{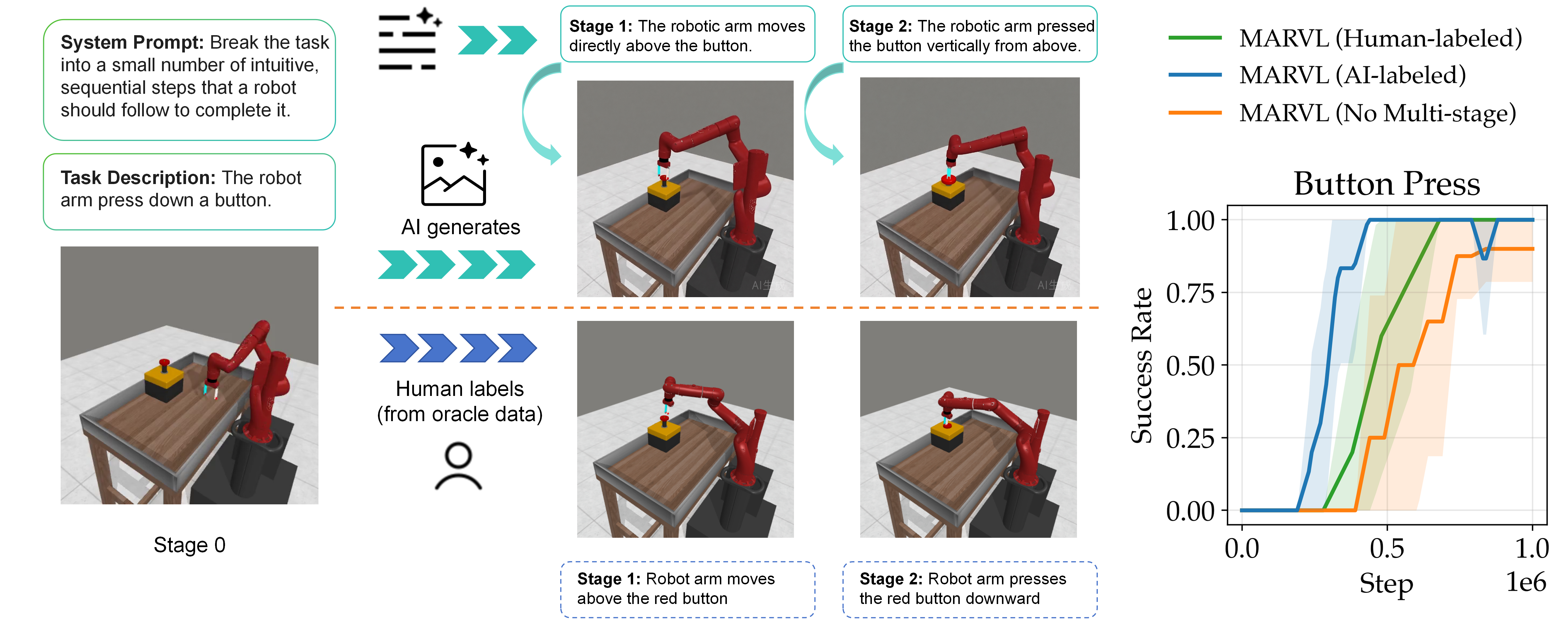}
  \caption{\textbf{Automated Sub-goal Generation Pipeline and Performance Comparison.} \textbf{Left:} The AI pipeline utilizes an LLM to decompose task instructions into textual sub-stages, which are then visualized by an image generation model to serve as intermediate visual goals. \textbf{Right:} Success rates on the Button Press task demonstrate that MARVL with AI-labeled goals (blue) achieves comparable or even superior performance to human-labeled goals (green), verifying the scalability and effectiveness of our generative pipeline.}
  \label{fig:ai}
\end{figure}

\begin{figure}
    \centering
    \includegraphics[width=\textwidth]{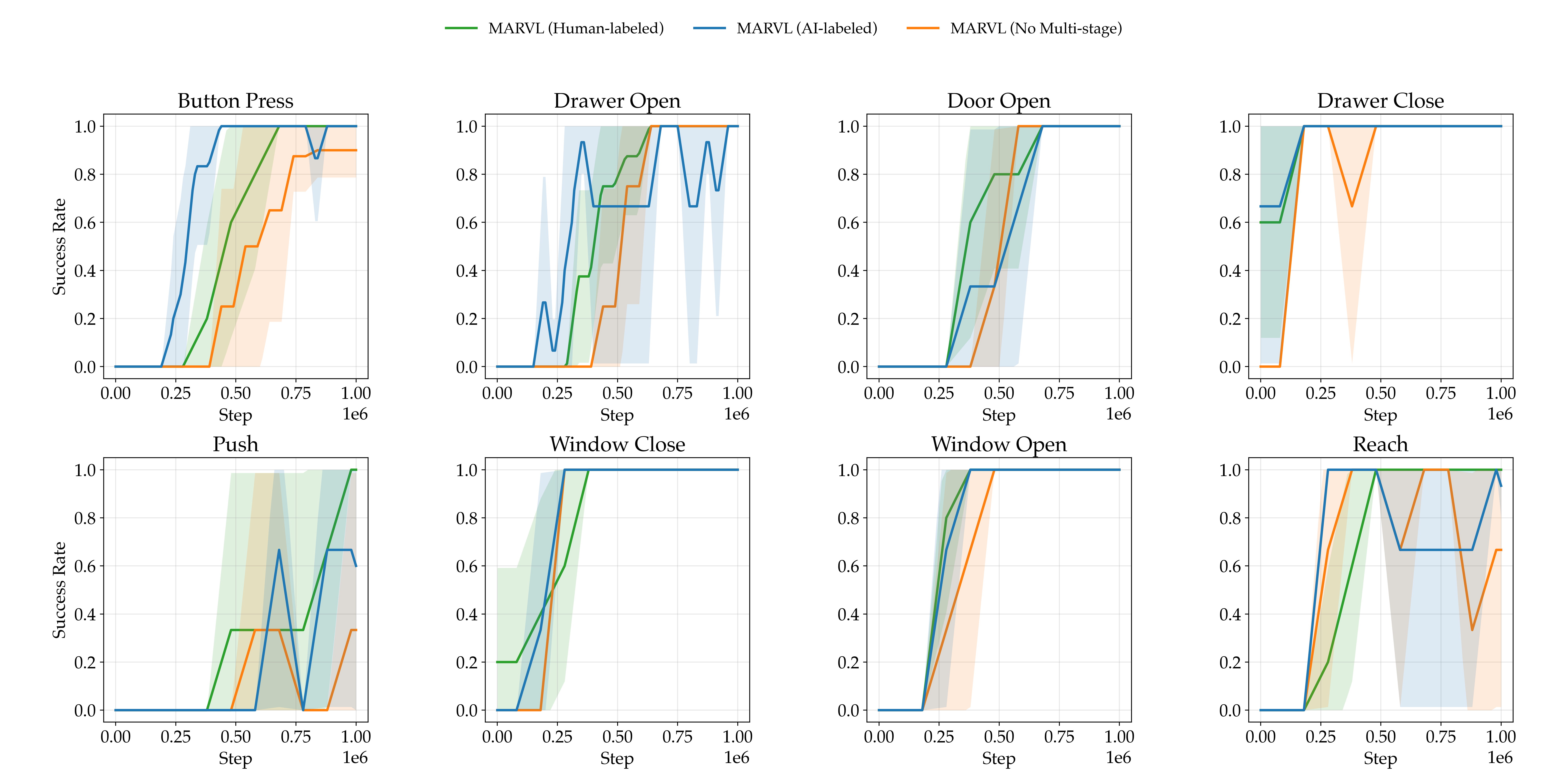}
    \caption{\textbf{Further Experiments on the Effectiveness of the Automated Sub-goal Generation Pipeline.}}
    \label{fig:ai_compa}
\end{figure}

\section{Fine-tuning Implementation Details}

\subsection{Data Collection}
To construct multi-view paired data for scene--view disentanglement, we uniformly sample frames from random trajectories across tasks at different time steps.
Using the MuJoCo interface, we programmatically perturb camera parameters in Meta-World to render each sampled scene from multiple viewpoints, while automatically recording the corresponding scene identity (Scene ID) and camera configuration (View ID).
Importantly, this procedure yields a lightweight dataset of \textbf{500 images}, which are randomly shuffled and recombined to form training pairs during fine-tuning.
No additional interaction, expert demonstrations, or target-domain data are required.
Figure~\ref{fig:robot_multi_view} illustrates representative multi-view examples from the collected dataset.

\subsection{Architecture}
The Scene Encoder $E_s$, View Encoder $E_v$, and Text Encoder are initialized with pre-trained CLIP ViT-B/32 \cite{radford2021clip} weights from laion2b\_s34b\_b79k \cite{schuhmann2022laion}. They accept normalized $224 \times 224$ RGB images or task text descriptions as input, yielding 512-dimensional feature vectors.The Decoder is a lightweight network (approx. 10\% of encoder parameters) based on transposed convolutions. It concatenates scene and view embeddings, followed by a projection layer and five transposed convolution blocks with Batch Normalization and CBAM \cite{woo2018cbam} attention  to reconstruct $224 \times 224$ RGB images. Detailed layer specifications are provided in Table \ref{tab:finetune}.

\begin{figure}[t!]
  \centering
  \includegraphics[width=\linewidth]{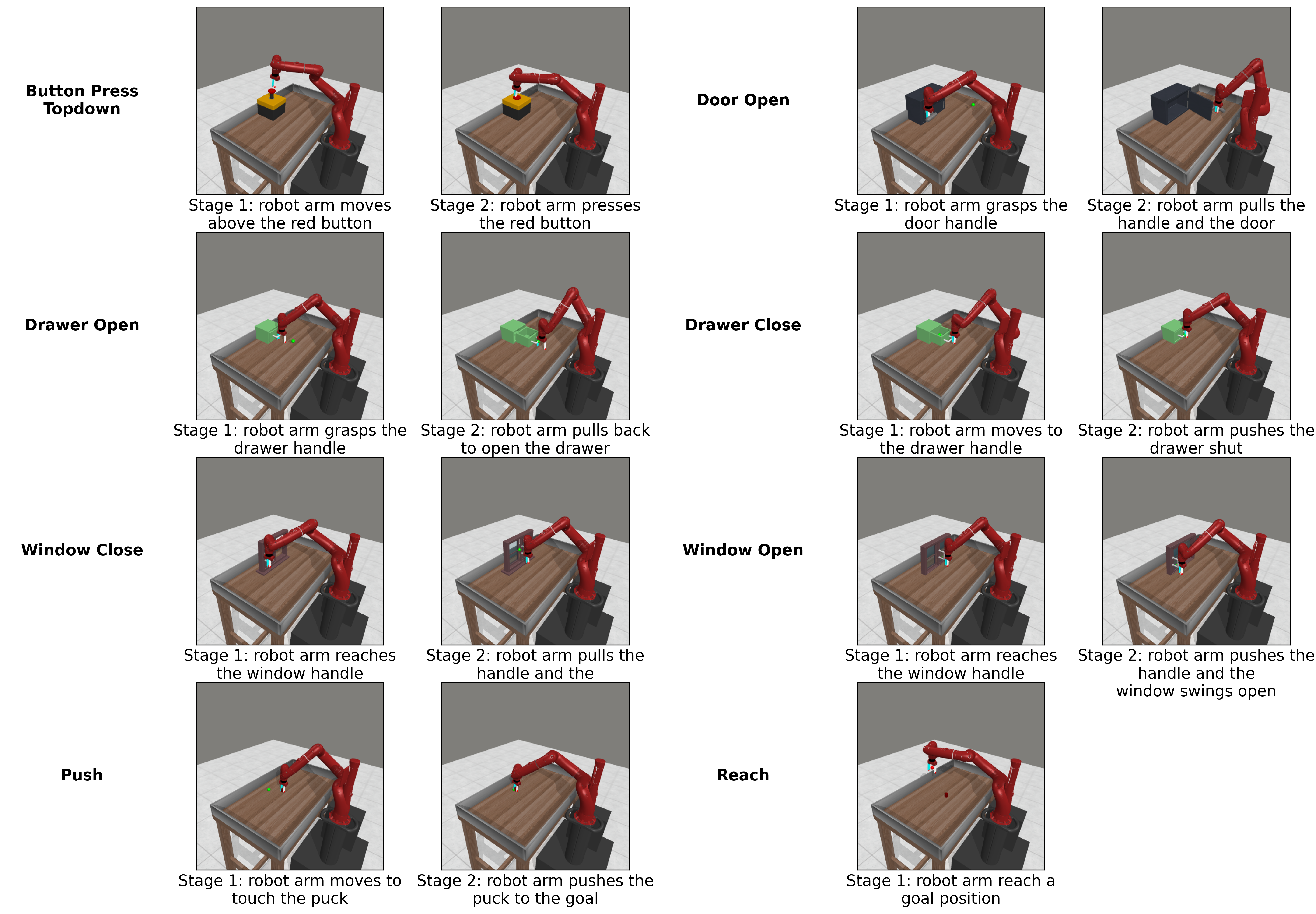}
  \caption{\textbf{Multi-Stage Target Image and Text Instruction}}
  \label{fig:multistage}
\end{figure}

\subsection{Process and Loss function}
The fine-tuning process proceeds in three substages. \textbf{Stage 1:} Both encoders remain frozen, and the decoder is optimized solely via $\mathcal{L}_{recon}$. \textbf{Stage 2:} We unfreeze the last four Transformer blocks of the ViT backbone and train using the full loss function. \textbf{Stage 3:} We progressively increase $\lambda_{clip}$ to ensure encoders retains alignment with the original CLIP text embedding space.

We provide a detailed formulation of the fine-tuning loss function below. Given a batch of sampled pairs of scenes $\{s_i, s_j\}$ and viewpoints $\{v_m, v_n\}$, Let $o_{s_i}^{v_m}$ denote an observation image of Scene $i$ captured from View $m$. The scene embedding is denoted as $z_s^{(i,m)} = E_s(o_{s_i}^{v_m})$ and the view embedding is denoted as $z_v^{(i,m)} = E_v(o_{s_i}^{v_m})$. The Decoder $D(z_s, z_v)$ reconstructs an image $\hat{o}_{s_i}^{v_k}$ from these embeddings.
$$\mathcal{L} = \mathcal{L}_{recon} + \lambda_{shuffle} \mathcal{L}_{shuffle} + \lambda_{consistency} \mathcal{L}_{consistency} + \lambda_{clip} \mathcal{L}_{clip}$$
\textbf{Reconstruction Loss $\mathcal{L}_{recon}$} To ensure feature completeness, 
we require the decoder $D$ to reconstruct the original image $o_{s_i}^{v_m}$ by combining the scene embedding $z_{s_i}^{(i,n)}$ (different view $n$) with the view embedding $z_{v_m}^{(j,m)}$ (different scene $j$), thereby forcing the model to learn scene-view decomposition by cross-reconstruction.
$$\hat{o}_{s_i}^{v_m} = D(z_{s_i}^{(i,n)}, z_{v_m}^{(j,m)}) $$
$$\mathcal{L}_{recon} =  \| \hat{o} - o \|_1 + \lambda_{lpips} \mathcal{L}_{LPIPS}(\hat{o}, o) $$where $\mathcal{L}_{LPIPS}$ denotes the Learned Perceptual Image Patch Similarity metric \cite{zhang2018unreasonable}, employed to capture high-frequency perceptual details.
\begin{figure}[ht] 
    \centering       
    \includegraphics[width=\linewidth]{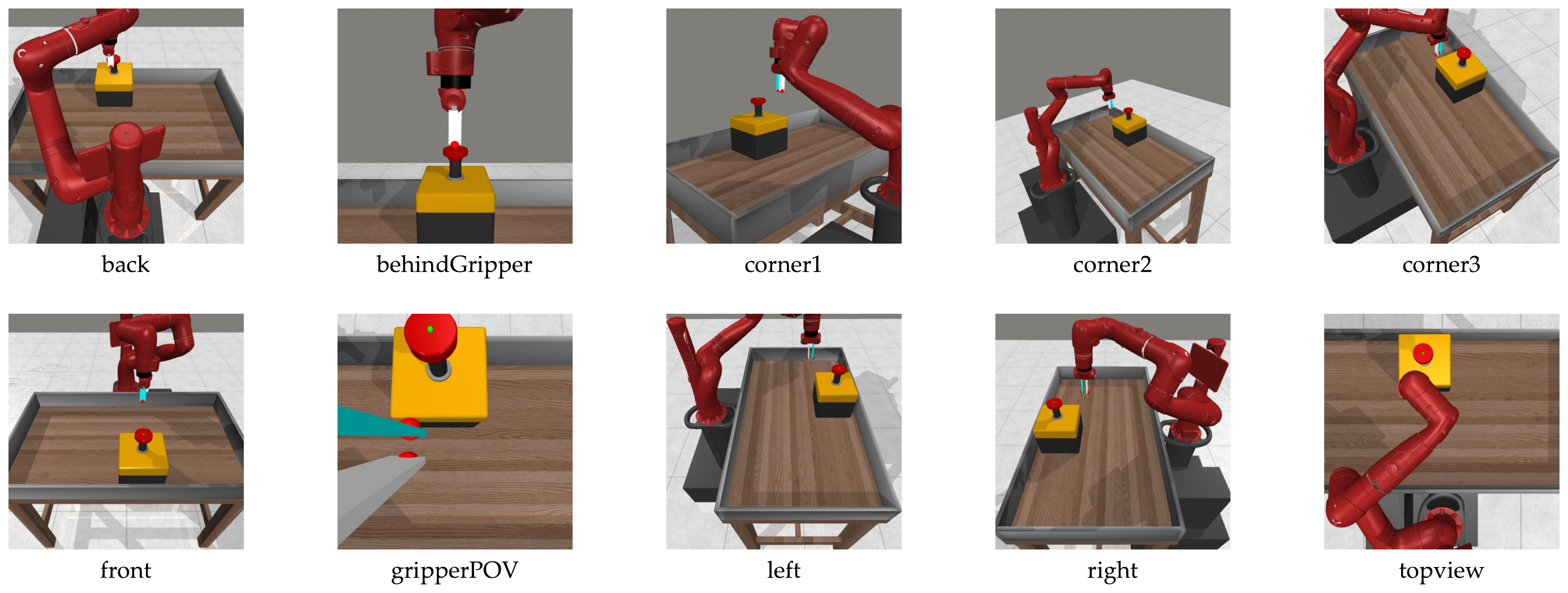} 
    \caption{\textbf{Multi-View Robotic Arm Image Examples in Button-Press Task}}
    \label{fig:robot_multi_view} 
\end{figure}

\textbf{Disentanglement Shuffle Loss $\mathcal{L}_{shuffle}$} This loss explicitly enforces orthogonality between $E_s$ and $E_v$ via latent variable swapping:
\begin{itemize}[leftmargin=*, itemsep=0pt]
    \item \textbf{Scene Invariance:} Swapping scene embeddings between different views of the same scene ($z_{s_i}^{(i,m)} \leftrightarrow z_{s_i}^{(i,n)}$) should yield invariant reconstructions, as the view code remains fixed.
    \item \textbf{View Invariance:} Swapping view embeddings between different scenes under the same viewpoint ($z_{v_m}^{(i,m)} \leftrightarrow z_{v_m}^{(j,m)}$) should similarly preserve the reconstruction, assuming the view code captures only perspective information.
\end{itemize}
Define $\hat{o}_{base} = D(z_{s_i}^{(i,m)}, z_{v_m}^{(i,m)})$ as the baseline reconstruction. The shuffle loss is defined as:$$\mathcal{L}_{shuffle} = \underbrace{\| D(z_{s_i}^{(i,n)}, z_{v_m}^{(i,m)}) - \hat{o}_{base} \|_2^2}_{\text{Scene Swap}} + \underbrace{\| D(z_{s_i}^{(i,m)}, z_{v_m}^{(j,m)}) - \hat{o}_{base} \|_2^2}_{\text{View Swap}}$$
Notably, we align the reconstruction $\hat{o}_{base}$ rather than the ground-truth image $o_{s_i}^{v_m}$. This design choice isolates the disentanglement error from the decoder's inherent reconstruction imperfections.

\textbf{Feature Consistency Loss $\mathcal{L}_{consistency}$ } We impose direct constraints within the feature space to ensure that scene embeddings are insensitive to viewpoint shifts and view embeddings are insensitive to scene content variations:$$\mathcal{L}_{consistency} = \underbrace{\left( 1 - \cos(z_{s_i}^{(i,m)}, z_{s_i}^{(i,n)}) \right)}_{\text{Scene Consistency}} + \underbrace{\left( 1 - \cos(z_{v_m}^{(i,m)}, z_{v_m}^{(j,m)}) \right)}_{\text{View Consistency}}$$

\textbf{Semantic Preservation Loss $\mathcal{L}_{clip}$ } To prevent catastrophic forgetting and preserve the pre-trained semantic alignment of the VLM, we employ an InfoNCE \cite{oord2018representation} contrastive loss. This objective aligns the fine-tuned scene embedding $z_s$ with its corresponding textual instruction embedding $t_{i}$ while separating it from irrelevant texts:$$\mathcal{L}_{clip} = -\log \frac{\exp(z_{s_i} \cdot t_{i} / \tau)}{\sum_{k} \exp(z_{s_i} \cdot t_{k} / \tau)}$$

Collectively, these terms establish a synergistic optimization landscape: while $\mathcal{L}_{recon}$ ensures the latent representations retain sufficient information for physical state recovery, $\mathcal{L}_{shuffle}$ and $\mathcal{L}_{consistency}$ act as structural regularizers that aggressively penalize the leakage of view nuisance factors into the scene embedding. Finally, $\mathcal{L}_{clip}$ serves as a semantic anchor, preventing the fine-tuned manifold from drifting away from the pre-trained linguistic alignment essential for instruction following.
\section{Theoretical Analysis of Task Direction Projection}
This section provides a rigorous theoretical analysis of the \textbf{Task Direction Projection (TDP) } mechanism introduced in Section 3.2. We model the embedding geometry to demonstrate that the proposed projection operator significantly enhances the Signal-to-Noise Ratio (SNR) and restores the monotonicity of the reward signal with respect to task progress. The mathematical symbols used in the following text follow the usage in the main body.


\begin{assumption}
\label{ass}
Define the scalar projected coordinate
\[
c_t = \frac{\langle e_{o_t}-e_{o_{\mathrm{start}}},\, d_{\mathrm{img}}\rangle}{\|d_{\mathrm{img}}\|^2}.
\]
We posit that within a sub-stage, $c_t$ serves as a monotonic proxy for the latent task progress $\lambda_t$ in expectation. Formally, there exists an increasing function $h$ such that
\[
\mathbb{E}[c_t \mid \lambda_t] = h(\lambda_t), \qquad h'(\cdot)>0.
\]
\textbf{Remark.} It is worth noting that for complex tasks, relying solely on a single terminal image and task description often fails to accurately characterize progress, potentially violating this monotonicity assumption (e.g., necessary intermediate states may appear visually or semantically distant from the final goal). To address this limitation and ensure the validity, we decompose the global task into sequential sub-stages.
\end{assumption}


\begin{theorem}
\label{the2}
    For any observation $e_{o_t}$, the projection operator $P_{d_{\text{img}}}$ preserves the task-aligned component while linearly attenuating the orthogonal nuisance component by a factor of $(1-\alpha)$. Consequently, this mechanism enhances the Signal-to-Noise Ratio (SNR) of the embedding by a factor of $(1-\alpha)^{-2}$.
\end{theorem}

\begin{proof}
    Let the relative displacement be $\Delta e_t = e_{o_t} - e_{o_{\text{start}}}$. We decompose $\Delta e_t$ into a component along $d_{\text{img}}$ and an orthogonal residual noise $\epsilon_t$:
$$\Delta e_t = c_t d_{\text{img}} + \epsilon_t, \quad \text{where } c_t = \frac{\langle \Delta e_t, d_{\text{img}} \rangle}{\|d_{\text{img}}\|_2^2} \text{ and } \epsilon_t \perp d_{\text{img}}. $$
Substituting into the projection definition, defining 
$$P_d(x) = \left( \alpha \Pi_d + (1-\alpha)I \right)x, \quad \text{where } \Pi_d = \frac{dd^\top}{\|d\|_2^2}$$
and using the properties $\Pi_{d_{\text{img}}}(d_{\text{img}}) = d_{\text{img}}$ and $\Pi_{d_{\text{img}}}(\epsilon_t) = 0$:

$$\begin{aligned}
e'_{o_t} &= P_{d_{\text{img}}}(c_t d_{\text{img}} + \epsilon_t) + e_{o_{\text{start}}} \\
&= \left[ \alpha \Pi_{d_{\text{img}}}(c_t d_{\text{img}} + \epsilon_t) + (1-\alpha)(c_t d_{\text{img}} + \epsilon_t) \right] + e_{o_{\text{start}}} \\
&= \left[ \alpha c_t d_{\text{img}} + (1-\alpha)c_t d_{\text{img}} + (1-\alpha)\epsilon_t \right] + e_{o_{\text{start}}} \\
&= e_{o_{\text{start}}} + c_t d_{\text{img}} + (1-\alpha)\epsilon_t.
\end{aligned}$$

This result demonstrates that the signal component $c_t d_{\text{img}}$ remains invariant, while the noise vector is scaled by $(1-\alpha)$.
Defining the SNR as the ratio of signal energy to noise energy, the raw SNR is proportional to $1/\mathbb{E}[\|\epsilon_t\|^2]$. In the projected space, the noise energy becomes $\mathbb{E}[\|(1-\alpha)\epsilon_t\|^2] = (1-\alpha)^2 \mathbb{E}[\|\epsilon_t\|^2]$. Thus, the SNR gain is given by the ratio $\frac{1}{(1-\alpha)^2}$, which grows significantly as $\alpha \to 1$.
\end{proof}

\textbf{Remark.} It is well-established that CLIP embeddings are $L_2$-normalized during training, constraining them to a unit hypersphere $\mathbb{S}^{n-1} = \{ e \in \mathbb{R}^d \mid \|e\|_2 = 1 \}$. While the TDP operator $P_{d_{\text{img}}}$ is a linear projection in $\mathbb{R}^d$ that may map the resulting embedding $e'_{o_t}$ off the hypersphere ($\|e'_{o_t}\|_2 \neq 1$), the TDP mechanism remains geometrically consistent with the VLM's inductive bias for two reasons:\begin{enumerate}\item \textbf{Scale Invariance of Reward:} The cosine similarity $\cos(A, B) = \frac{\langle A, B \rangle}{\|A\| \|B\|}$ is a homogeneous function of degree zero. Formally, for any scalar $\kappa > 0$, $\cos(\kappa A, B) = \cos(A, B)$. Therefore, the norm-changing effect of the projection does not shift the reward value; only the directional alignment is preserved and denoised.\item \textbf{Tangent Space Approximation:} In high-dimensional manifolds, for a localized sub-stage, the geodesic distance on the hypersphere is well-approximated by the Euclidean distance in the tangent space at $e_{o_{\mathrm{start}}}$. TDP can be viewed as performing directional filtering in this local tangent space before the implicit re-normalization performed by the cosine similarity.\end{enumerate}This ensures that TDP effectively filters noise in the embedding space while respecting the spherical topology of the VLM's representation.



\begin{theorem}
\label{the3}
    While raw cosine similarity is susceptible to non-monotonicity induced by noise norm fluctuations, under the condition that the visual task direction aligns sufficiently with the textual instruction in the embedding space, TDP ensures reward $\tilde{r}$ has positive gradients with respect to projected coordinate $c_t$, which serves as a monotonic proxy for the latent task progress $\lambda_t$.
\end{theorem}

\begin{proof}
    we decompose the observation embedding into a deterministic signal  and random noise:$$e'_{o_t} = \underbrace{e_{o_{\text{start}}} + c_t d_{\text{img}}}_{S(c_t)} + \underbrace{(1-\alpha)\epsilon_t}_{N_t}$$

where $S(c_t)$ is the signal vector evolving with projected coordinate $c_t$, and $N_t$ is the residual noise. In the raw setting ($\alpha=0$), $N_t = \epsilon_t$; in the TDP setting ($\alpha \approx 1$), $N_t \to 0$.The reward function can be written as $$\tilde{r}_t(c_t, N_t) = \cos(e'_{o_t}, e'_l) = \frac{\langle S(c_t) + N_t, e'_l \rangle}{\|S(c_t) + N_t\| \|e'_l\|}$$
To rigorously analyze how noise disrupts the monotonicity of the reward signal, we consider the total change $d\tilde{r}_t$ with respect to time step $t$. When the agent makes infinitesimal task progress $dc_t > 0$, it is accompanied by a random perturbation of the noise vector $dN_t$:$$d\tilde{r}_t = \underbrace{\frac{\partial \tilde{r}_t}{\partial c_t} dc_t}_{\text{Signal Gain}} + \underbrace{\langle \nabla_{N_t} \tilde{r}_t, dN_t \rangle}_{\text{Noise Interference}}$$
$$\begin{aligned} \nabla_{N_t} \tilde{r}_t = \nabla_{e'_{o_t}} \tilde{r}_t &= \frac{\|e'_{o_t}\| \|e'_l\| \cdot e'_l - \langle e'_{o_t}, e'_l \rangle \cdot \|e'_l\| \frac{e'_{o_t}}{\|e'_{o_t}\|}}{\|e_{o_t}\|^2 \|e'_l\|^2} \\
&= \frac{1}{\|e'_{o_t}\|} \left( \frac{e'_l}{\|e'_l\|} - \tilde{r}_t \frac{e'_{o_t}}{\|e'_{o_t}\|} \right) \\
&= \frac{1}{\|e'_{o_t}\|} (\hat{e}_l - \tilde{r}_t \hat{e}_{o_t}) \end{aligned} $$

Analyze a common worst-case scenario: noise magnitude expansion. Assuming that at a certain moment, the direction of noise change aligns with the noise vector itself, i.e., $dN_t$ is parallel to $N_t$ (e.g., background clutter increases), denoted as $dN_t = \gamma \hat{N}_t$ where $\gamma > 0$. In this case, the noise interference term becomes: $$\text{Noise Interference} = \langle \nabla_{N_t} \tilde{r}_t, dN_t \rangle = \frac{\gamma}{\|e'_{o_t}\|} \left( \langle \hat{e}_l, \hat{N}_t \rangle - \tilde{r}_t \langle \hat{e}_{o_t}, \hat{N}_t \rangle \right)$$
In high-dimensional semantic spaces, the random noise vector $N_t$ is likely orthogonal to the fixed task goal $e_l$, $\langle \hat{e}_l, \hat{N}_t \rangle \approx 0$. For second term $\langle \hat{e}_{o_t}, \hat{N}_t \rangle$, expanding $\hat{e}_{o_t} = \frac{S(c_t) + N_t}{\|e'_{o_t}\|}$, we get:$$\langle \hat{e}_{o_t}, \hat{N}_t \rangle = \frac{\langle S(c_t), N_t \rangle + \|N_t\|^2}{\|e'_{o_t}\| \|N_t\|}$$Since the signal is orthogonal to the noise ($\langle S, N_t \rangle \approx 0$), this term approximates to $\frac{\|N_t\|}{\|e'_{o_t}\|}$, which is a strictly positive value. $$\text{Noise Interference} \approx \frac{\gamma}{\|e'_{o_t}\|} \left( 0 - \tilde{r}_t \frac{\|N_t\|}{\|e'_{o_t}\|} \right) = - \gamma \tilde{r}_t \frac{\|N_t\|}{\|e'_{o_t}\|^2}$$
When the noise magnitude increases ($\gamma > 0$) and is sufficiently large, the negative noise interference term dominates the positive signal gain term:$$\left| - \gamma \tilde{r}_t \frac{\|N_t\|}{\|e'_{o_t}\|^2} \right| > \left| \frac{\partial \tilde{r}_t}{\partial c_t} dc_t \right| \implies d\tilde{r}_t < 0$$This mathematical derivation rigorously proves that even if the agent approachs the goal in high-dimensional semantic space($dc_t > 0$), the raw reward signal can decrease. This is due to the reward function's sensitivity to the full noise vector $N_t$ in the denominator, where the negative gradient component (stemming from the $-\tilde{r}_t \hat{e}_{o_t}$ term) overwhelms the marginal signal improvement. This explains the severe oscillations and lack of monotonicity observed in Raw VLM rewards during experiments.

In TDP setting ($\alpha \to 1$), we have $\|N_t\| \to 0$ and $\mathrm{d}\|N_t\| \to 0$. Consequently, the reward function degenerates to a deterministic form:$$\lim_{\alpha \to 1} \tilde{r}_t = \frac{\langle S(c_t), e'_l \rangle}{\|S(c_t)\| \|e'_l\|} = \frac{\langle e_{o_{\text{start}}} + c_t d_{\text{img}}, e'_l \rangle}{\|e_{o_{\text{start}}} + c_t d_{\text{img}}\| \|e'_l\|}$$
Note that $\frac{\partial S(c_t)}{\partial c_t} = d_{img}$:$$\begin{aligned}
\frac{\partial \tilde{r}}{\partial c_t} &= \frac{1}{\|e_l'\|} \frac{\partial}{\partial c_t} \left( \langle S(c_t), e_l' \rangle \|S(c_t)\|^{-1} \right) \\
&= \frac{1}{\|e_l'\|} \left( \langle d_{img}, e_l' \rangle \|S(c_t)\|^{-1} + \langle S(c_t), e_l' \rangle \cdot (-1)\|S(c_t)\|^{-2} \frac{\partial \|S(c_t)\|}{\partial c_t} \right)
\end{aligned}$$The derivative of the norm is $\frac{\partial \|S(c_t)\|}{\partial c_t} = \frac{\langle S(c_t), d_{img} \rangle}{\|S(c_t)\|}$. Substituting this back:$$\begin{aligned}
\frac{\partial \tilde{r}}{\partial c_t} &= \frac{1}{\|e_l'\| \|S(c_t)\|^3} \left[ \langle d_{img}, e_l' \rangle \|S(c_t)\|^2 - \langle S(c_t), e_l' \rangle \langle S(c_t), d_{img} \rangle \right]
\end{aligned}$$By normalizing the terms with $\|d_{img}\|\|e'_l\|\|S(c_t)\|^2$, we rewrite the positivity condition in its angular form:$$\frac{\langle d_{img}, e_l' \rangle}{\|d_{img}\| \|e_l'\|} > \frac{\langle S(c_t), e_l' \rangle}{\|S(c_t)\| \|e_l'\|} \cdot \frac{\langle S(c_t), d_{img} \rangle}{\|S(c_t)\| \|d_{img}\|}$$$$\implies \cos(d_{img}, e_l') > \cos(S(c_t), e_l') \cdot \cos(S(c_t), d_{img})$$ This condition necessitates that the visual task direction $d_{img}$ aligns with the textual target $e_l'$. Since $e_l' = \mathcal{P}_{d_{text}}(e_l)$, $e_l'$ is projected onto the task direction $d_{img}$, while $S_t = e_{o_{\mathrm{start}}}+c_t d_{\mathrm{img}}$ denotes the ideal observation embedding obtained by shifting $e_{o_{\mathrm{start}}}$ along the task direction, and in general $S_t$ is not colinear with $d_{\mathrm{img}}$. For the expected VLM, the LHS dominating the RHS product of intermediate alignments, consequently the positive gradient property is maintained. Coupled with \textbf{Assumption \ref{ass} }, this restores the critical condition of monotonicity between the reward $\tilde{r}$ and latent task progress $\lambda_t$, effectively isolating the trend from background noise, theoretically validating the consistency of TDP rewards. The empirical validation in \textbf{Assumption ~\ref{tdp}} further corroborates the validity of \textbf{Assumption \ref{ass}} and the implications of \textbf{Theorem \ref{the3}} .
\end{proof}


\section{Additional Results}
\subsection{Empirical Validation of Task Direction Projection}
\label{tdp}

Figure ~\ref{fig:rew} compares the raw CLIP reward against the TDP-processed reward on oracle trajectories. The raw reward (orange) exhibits significant volatility and non-monotonicity, often decreasing as the agent approaches the goal. In contrast, the TDP reward (blue) demonstrates a monotonic progression aligned with the ground-truth task execution, empirically validating the noise suppression analysis in \textbf{Theorem \ref{the2}} and \textbf{\ref{the3}}.

\begin{figure}[h]
  \centering
  \includegraphics[width=\linewidth]{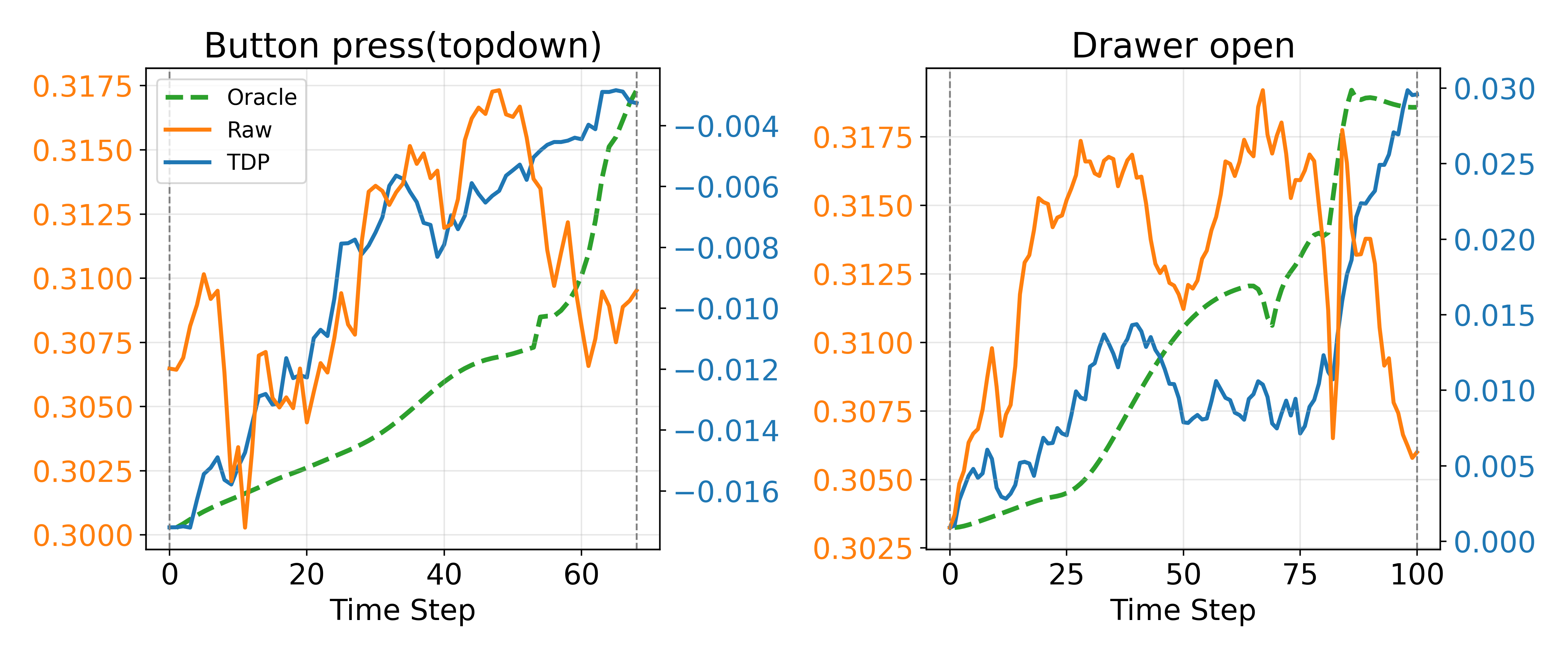}
  \caption{\textbf{Comparison of Reward Signal Quality on Oracle Trajectories.}}
  \label{fig:rew}
\end{figure}

\subsection{Robustness of Scene Encoder to Viewpoint Shifts}
 We evaluate reward consistency by randomly switching camera viewpoints every 5 steps along an oracle trajectory. As shown in Figure ~\ref{fig:shift}, the raw VLM reward fluctuates drastically with viewpoint changes. However, by effectively filtering out irrelevant view information, the reward curve derived from our fine-tuned Scene Encoder remains consistent despite these abrupt visual transitions, recovers signals that faithfully reflects the true task progress. This invariance is particularly critical for real-world robotic deployment, where reward systems must remain robust to unexpected camera jitter or calibration errors.
\begin{figure}[h]
  \centering
  \includegraphics[width=\linewidth]{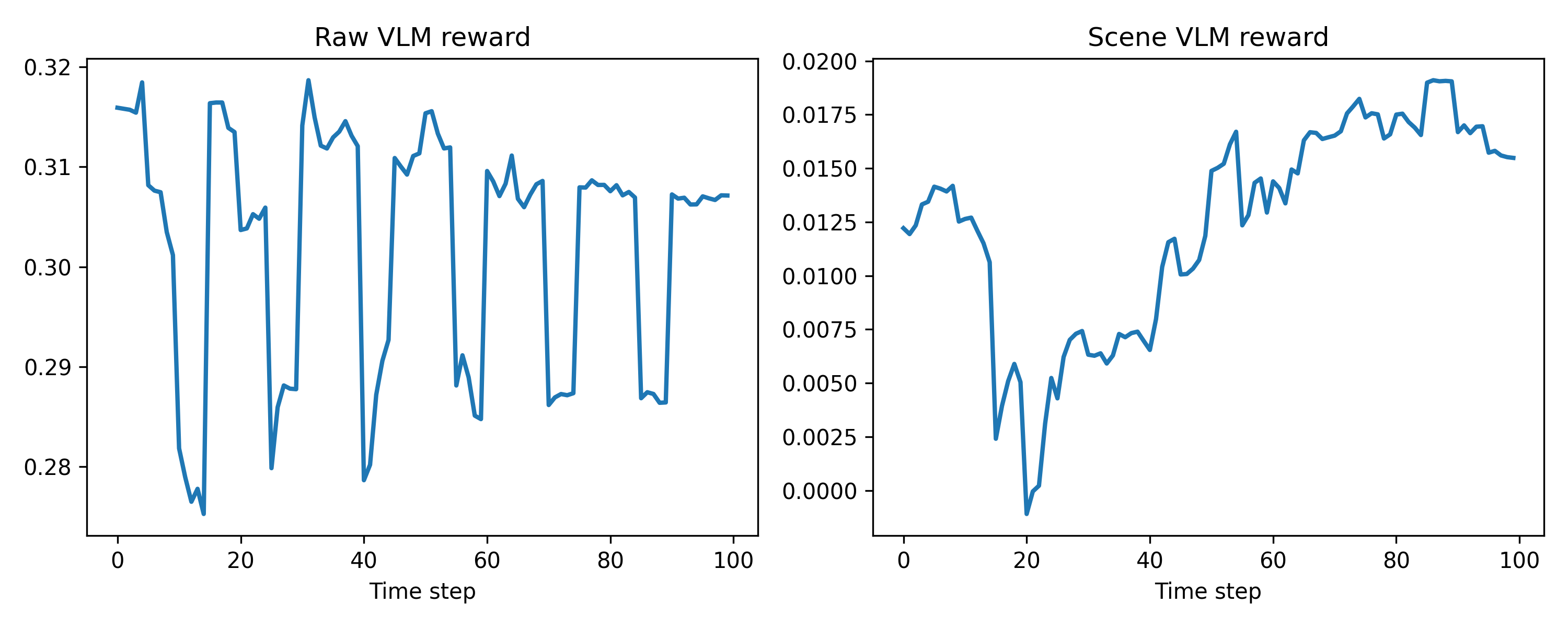}
  \caption{\textbf{Robustness to Camera Viewpoint Shifts}}
  \label{fig:shift}
\end{figure}

\subsection{Analysis of Confidence-Thresholded Shaping}
In addition to the confidence-thresholded shaping (CTS) mechanism used throughout the main experiments, we investigate two alternative strategies for suppressing noise in VLM-based rewards: exponential moving average (EMA) smoothing and Kalman filtering. Both methods aim to stabilize reward signals by attenuating high-frequency fluctuations, but differ in their inductive biases and responsiveness to sudden semantic alignment.

As shown in Figure~\ref{fig:filter_comparison}, all three filtering strategies effectively suppress spurious reward activations and ultimately enable successful learning. EMA filtering improves stability by smoothing short-term fluctuations, but this temporal averaging can delay the reward’s response to abrupt progress signals in certain tasks. Kalman filtering enforces stronger temporal coherence, yet exhibits higher sensitivity to its modeling assumptions and hyperparameters, leading to less consistent convergence behavior across tasks. In contrast, CTS applies a simple, distribution-aware gating mechanism that preserves sharp reward transitions once semantic confidence surpasses the noise floor, yielding more uniform convergence dynamics. Given its favorable balance between robustness, responsiveness, and implementation simplicity, we adopt CTS as the default choice in MARVL.

\begin{figure}[h!]
  \centering
  \includegraphics[width=0.8\linewidth]{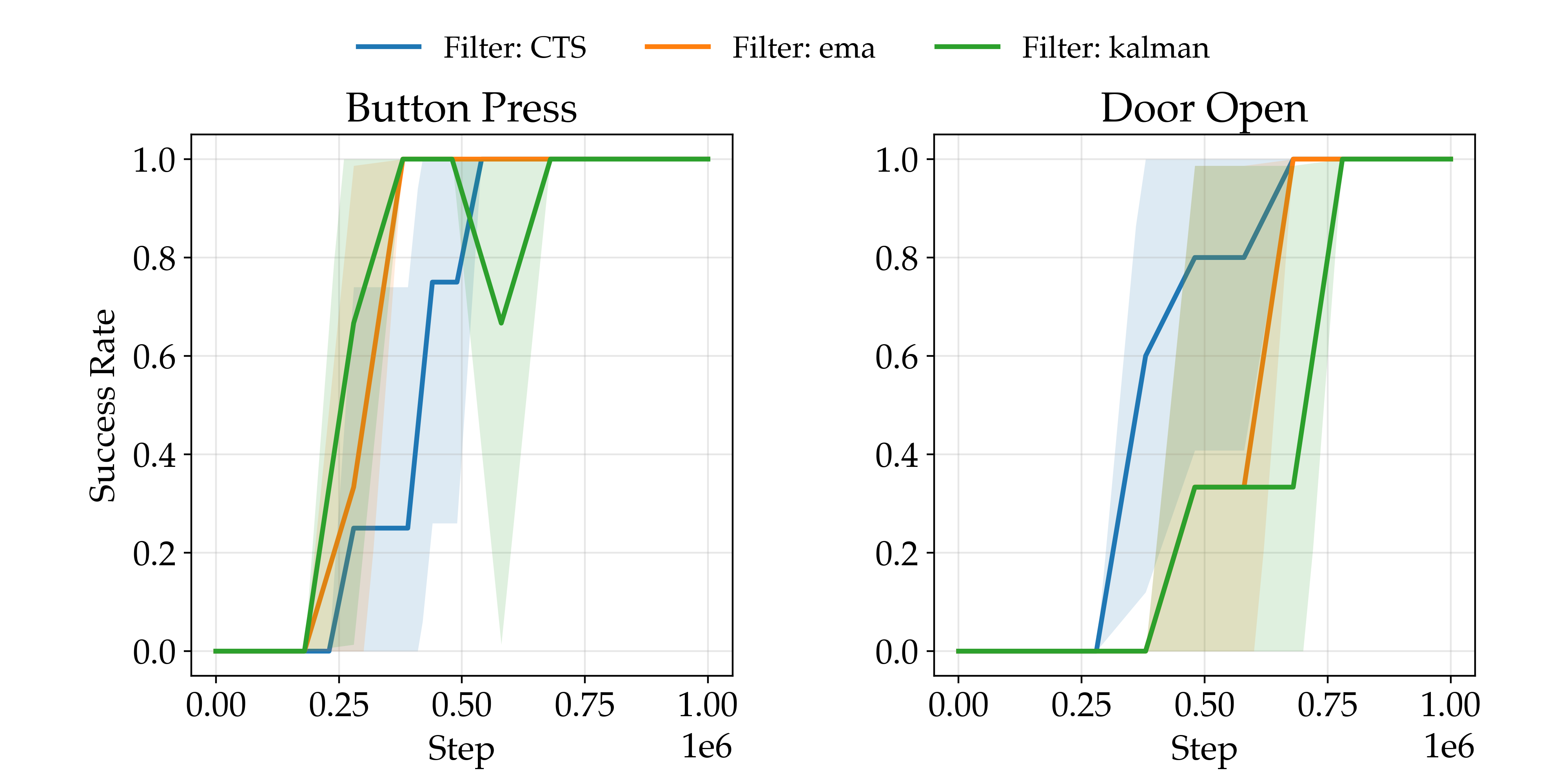}
  \caption{\textbf{Comparison of Different Confidence-Based Reward Filtering Strategies.}}
  \label{fig:filter_comparison}
\end{figure}

\subsection{Effect of View-Disentangled Fine-Tuning on Semantic Alignment}
\begin{figure}[h!]
  \centering
  \includegraphics[width=\linewidth]{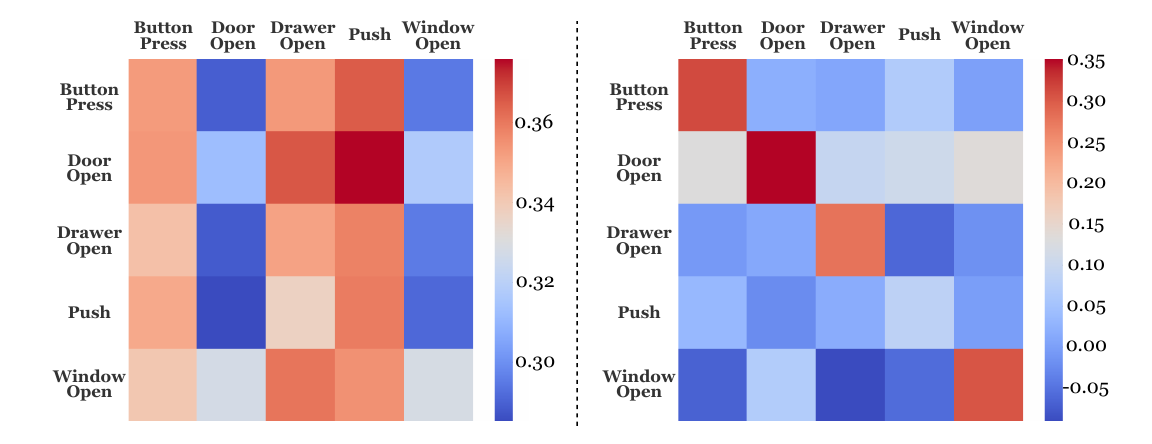}
  \caption{\textbf{Text–Image Similarity Matrices Before (left) and After (right) Scene–View Disentangled Fine-tuning.} Fine-tuning enhances diagonal dominance and task separability. Lower similarity for Push reflects scene sparsity rather than semantic confusion.}
  \label{fig:heatmap}
\end{figure}

Figure \ref{fig:heatmap} analyzes the effect of scene–view disentangled fine-tuning on text–image alignment. Compared to the original embedding space (Figure \ref{fig:heatmap}, left), the fine-tuned representations (Figure \ref{fig:heatmap}, right) exhibit a markedly stronger diagonal structure, indicating improved task separability. Each visual observation aligns most strongly with its corresponding instruction, while off-diagonal activations are substantially reduced. This confirms that removing view-dependent nuisance factors helps restore semantic coherence in the joint embedding space, directly addressing the semantic misalignment issue discussed in Section 2.

A consistent exception is observed for the Push task, whose similarity values remain low across all instructions. This behavior is expected: in the Push environment, the scene is visually sparse—aside from the robot arm, the workspace is largely empty—providing limited semantic cues for the VLM to associate with the instruction. Importantly, although the absolute similarity is low, the diagonal entry remains dominant, indicating preserved relative alignment rather than increased confusion.

These observations highlight both the effectiveness of view-disentangled fine-tuning for improving semantic alignment and the inherent difficulty of visually under-specified tasks, motivating the confidence-aware reward shaping introduced in Section 3.3.

\subsection{Effect of the Mixing Parameter $\alpha$ in TDP}

 \begin{figure}[h!]
  \centering
  \includegraphics[width=0.25\linewidth]{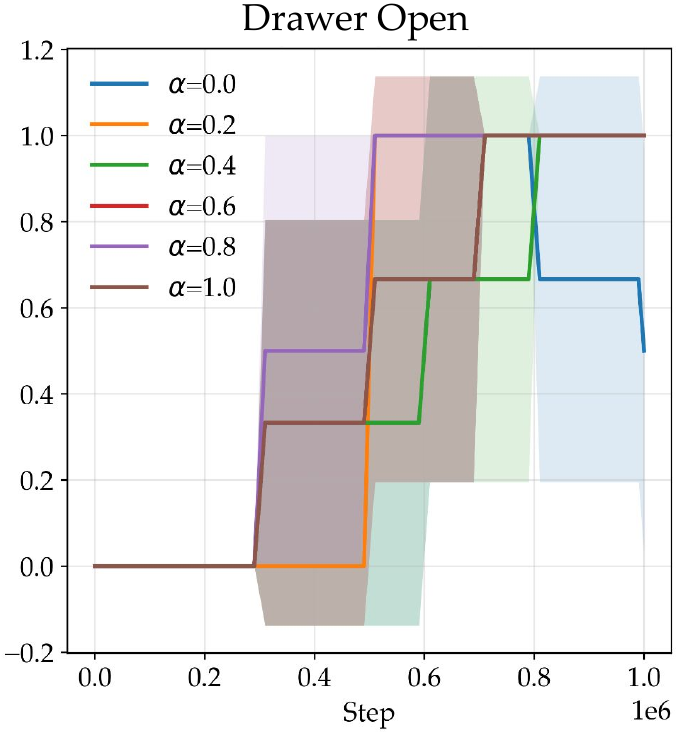}
  \caption{\textbf{Effect of the Mixing Parameter $\alpha$ in TDP.}}
  \label{fig:finalablation}
\end{figure}

We study the effect of the mixing parameter $\alpha$ in Task Direction Projection (TDP), which controls the interpolation between the raw visual embedding and the task-aligned projected component. Figure~\ref{fig:finalablation} reports the success rate on the Drawer Open task for different values of $\alpha$. Overall, intermediate values of $\alpha$ yield more reliable learning dynamics, while extreme settings tend to degrade stability. In particular, $\alpha = 0$ (no projection) results in slower convergence and reduced final performance, indicating that unstructured similarity signals are insufficient for consistent progress estimation. Conversely, overly large $\alpha$ can introduce premature or noisy progress signals, leading to higher variance across runs. These results suggest that a moderate balance between raw observations and task-aligned projection is important for stable training, and we therefore fix $\alpha$ to a constant intermediate value in all experiments.

\newpage
\begin{table}[htbp]
\centering
\begin{minipage}[t]{0.48\textwidth}
\centering
\caption{Hyperparameters for MARVL and Baselines}
\begin{tabular}{ll}
\toprule
\textbf{Parameter} & \textbf{Value} \\
\midrule
\multicolumn{2}{c}{\textit{Common Hyperparameters (SAC)}} \\
\midrule
VLM reward weight $\rho$ & 0.05 \\
Adam learning rate & $3 \times 10^{-4}$ \\
Batch size & 256 \\
Target network $\tau$ & 0.005 \\
Discount factor $\gamma$ & 0.99 \\
Hidden layer dimensions & (256, 256) \\
Initial random exploration steps & $10^4$ \\
Actor target entropy & $-\text{act\_dim} / 2$ \\
Actor/Critic min std scale & $1 \times 10^{-3}$\\ 
Action repeat & 10 \\
\midrule
\multicolumn{2}{c}{\textit{MARVL}} \\
\midrule
TDP projection blend $\alpha$ & 0.8 \\
quantile threshold $m$ & 0.97 \\
Logistic steepness parameter $\kappa$ & 100 \\

Stage transition similarity threshold & 0.997 \\
Stage transition patience & 4 \\
\midrule
\multicolumn{2}{c}{\textit{FuRL}} \\
\midrule
FuRL cosine margin & 0.25 \\
FuRL L2 distance margin & 0.25 \\
FuRL embedding buffer size & $2 \times 10^4$ \\
FuRL window gap $k$ & 10 \\
FuRL reward projection network & (256, 64) \\
\midrule
\multicolumn{2}{c}{\textit{TD3}} \\
\midrule
TD3 policy noise & 0.2 \\
TD3 noise clip & 0.5 \\
TD3 policy update frequency & 2 \\
TD3 exploration noise & 0.1 \\
\midrule
\multicolumn{2}{c}{\textit{Relay}} \\
\midrule
Relay success threshold & 2500 \\
Relay exploration noise & 0.2 \\
\midrule
\multicolumn{2}{c}{\textit{LIV }} \\
\midrule
LIV learning rate & $1 \times 10^{-5}$ \\
LIV weight decay & $1 \times 10^{-3}$ \\
LIV discount factor & 0.98 \\

%
\bottomrule
\end{tabular}
\label{tab:hyperparameters}
\end{minipage}
\hfill
\begin{minipage}[t]{0.48\textwidth}
\centering
\caption{Fine-tuning Hyperparameters}
\begin{tabular}{ll}
\toprule
\textbf{Parameter} & \textbf{Value} \\
\midrule
\multicolumn{2}{c}{\textit{Architecture}} \\
\midrule
Image Size & 224 \\
Patch Size & 14 \\
ViT Model & ViT-B-32 \\
Pre-trained Weights & laion2b\_s34b\_b79k \\
Encoder Output Dimension & 512 \\
Decoder Input Dimension & 1024 \\
Conv Channels & [256, 128, 64, 32] \\
Conv Kernel Size & 3 \\
Upsampling Mode & bilinear \\
Dropout Rate & 0.1 \\
\midrule
\multicolumn{2}{c}{\textit{Training}} \\
\midrule
Batch Size & 32 \\
Learning Rate & $1 \times 10^{-4}$ \\
Total Epochs & 90 \\
Stage-1 / Stage-2 / Stage-3 Epochs & 30 / 30 / 30 \\
\midrule
\multicolumn{2}{c}{\textit{Optimizer}} \\
\midrule
Weight Decay & 0.01 \\
Label Smoothing & 0.1 \\
Gradient Accumulation Steps & 2 \\
Max Gradient Norm & 1.0 \\
Scheduler Type & cosine \\
Warmup Epochs & 5 \\
Minimum Learning Rate & $1 \times 10^{-6}$ \\
\midrule
\multicolumn{2}{c}{\textit{Loss Weight Configuration}} \\
\midrule
LPIPS Weight & 0.05 \\
Shuffle Coefficient & 1.0 \\
Scene Consistency Coefficient & 1.0 \\
View Consistency Coefficient & 0.25 \\
Scene Contrastive Coefficient & 1.0 \\
View Contrastive Coefficient & 0.25 \\
Clip Coefficient & [0, 0.5, 1] \\
Temperature & 0.07 \\
\bottomrule
\end{tabular}
\label{tab:finetune}
\end{minipage}
\end{table}


\end{document}